\documentclass[10pt,twocolumn,letterpaper]{article}

\usepackage{cvpr}
\usepackage{times}
\usepackage{epsfig}
\usepackage{graphicx}
\usepackage{amsmath}
\usepackage{amssymb}
\usepackage{algorithm}
\usepackage{algorithmic}
\usepackage[ruled,vlined,algo2e]{algorithm2e}

\usepackage{helvet} 
\usepackage{courier}  
\usepackage[numbers]{natbib}
\usepackage[utf8]{inputenc} 
\usepackage[T1]{fontenc}    

\usepackage{booktabs}       
\usepackage{amsfonts}       
\usepackage{nicefrac}       
\usepackage{microtype}      

\usepackage{epsfig}
\usepackage{amsmath}
\usepackage{amssymb}
\usepackage{adjustbox}

\usepackage{amsthm}  
\usepackage{wasysym}
\usepackage[ruled,vlined,algo2e]{algorithm2e}
\providecommand{\SetAlgoLined}{\SetLine}
\usepackage{color}
\usepackage{dsfont}	
\usepackage{wrapfig}
\usepackage{footnote}
\usepackage{epstopdf}
 \usepackage{enumitem}
\usepackage{subfigure}
\usepackage{caption}
\usepackage{comment}
\usepackage{multirow}
\usepackage{algorithm}
\usepackage{algorithmic}

\def\eg{\emph{e.g., }}
\def\ie{\emph{i.e., }}

\def\vs{\emph{vs. }}
\def\wrt{\emph{w.r.t. }}
\def\etal{\emph{et al. }}

\usepackage{pifont}

\DeclareMathOperator*{\argmin}{arg\,min}

\usepackage{mathtools}

\makeatletter
\newcommand*{\rom}[1]{\expandafter\@slowromancap\romannumeral #1@}
\makeatother
\makeatletter
\newcommand\footnoteref[1]{\protected@xdef\@thefnmark{\ref{#1}}\@footnotemark}
\makeatother
\newcommand{\bfsection}[1]{\vspace*{0.1cm}\noindent\textbf{#1.}}

\newtheorem{prop}{Proposition}

\usepackage[pagebackref=true,breaklinks=true,letterpaper=true,colorlinks,bookmarks=false]{hyperref}

\cvprfinalcopy 

\def\cvprPaperID{4304} 

\ifcvprfinal\fi
\begin{document}

\title{Learning to Segment 3D Point Clouds in 2D Image Space}

\author{Yecheng Lyu\thanks{Part of this work was done when the author was an intern at Mitsubishi Electric Research Laboratories (MERL).} \hspace{1cm} Xinming Huang \hspace{1cm} Ziming Zhang \\
Worcester Polytechnic Institute\\
{\tt\small \{ylyu, xhuang, zzhang15\}@wpi.edu}
}

\maketitle

\begin{abstract}
In contrast to the literature where local patterns in 3D point clouds are captured by customized convolutional operators, in this paper we study the problem of how to effectively and efficiently project such point clouds into a 2D image space so that traditional 2D convolutional neural networks (CNNs) such as U-Net can be applied for segmentation. To this end, we are motivated by graph drawing and reformulate it as an integer programming problem to learn the topology-preserving graph-to-grid mapping for each individual point cloud. To accelerate the computation in practice, we further propose a novel hierarchical approximate algorithm. With the help of the Delaunay triangulation for graph construction from point clouds and a multi-scale U-Net for segmentation, we manage to demonstrate the state-of-the-art performance on ShapeNet and PartNet, respectively, with significant improvement over the literature. Code is available at \url{https://github.com/Zhang-VISLab}.

\end{abstract}

\section{Introduction}
Recently point cloud processing has been attracting more and more attention \cite{qi2017pointnet,Pham_2019_JSIS3D,hua2018pointwise,qi2017pointnet++,Duan_2019_SRN,wang2018sgpn,Liu_2019_RS_CNN,Komarichev_2019_A_CNN,Thomas_2019_KPConv,li2018so_net,Wu_2019_PointConv,li2018pointcnn,Zhao_2019_PointWeb,Zhang2019ShellNet,Liu_2019_DensePoint,Yu_2019_PartNet,huang2018recurrent,Meng_2019_VV_Net,le2018pointgrid,Mao_2019_Interpolated,wang2017cnn,Lei_2019_Octree,su2018splatnet,rao2019spherical,yi2017syncspeccnn}. As a fundamental data structure to store the geometric features, a point cloud saves the 3D positions of points scanned from the physical world as an orderless list. In contrast, images have regular patterns on 2D grid with well-organised pixels in local neighborhood. Such local regularity is beneficial for fast 2D convolution, leading to well-designed convolutional neural networks (CNNs) such as FCN \cite{long2015fully}, GoogleNet \cite{szegedy2015googlenet} and ResNet \cite{he2016resnet} that can efficiently and effectively extract local features from pixels to semantics with state-of-the-art performance for different applications.

\bfsection{Motivation}
In fact PointNet\footnote{For simplicity in our explanation, we assume no bias term in PointNet.} \cite{qi2017pointnet} for point cloud classification and segmentation can be re-interpreted from the perspective of CNN. In general, PointNet projects each 3D $(x,y,z)$-point into a higher dimensional feature space using a multilayer perceptron (MLP) and pools all the features from a cloud globally as a cloud signature for further usage. As an equivalent CNN implementation, one can construct an $(x,y,z)$-image with all the 3D points as the pixels in a random order and $(0,0,0)$ for the rest of the image, and apply $1\times1$ convolutional kernels sequentially to the image, followed by a global max-pooling operator. Different from conventional RGB images, here $(x,y,z)$-images define a new 2D image space with $x, y, z$ as channels. Same image representation has been explored in \cite{lyu2018real,lyu2018chipnet,milioto2019rangenet++,wu2017squeezeseg,wu2019squeezesegv2} for LiDAR points. Unlike CNNs, PointNet lacks of the ability of extracting local features that may limit its performance.

\begin{figure}[t]
	\centering
	\includegraphics[width=.95\columnwidth]{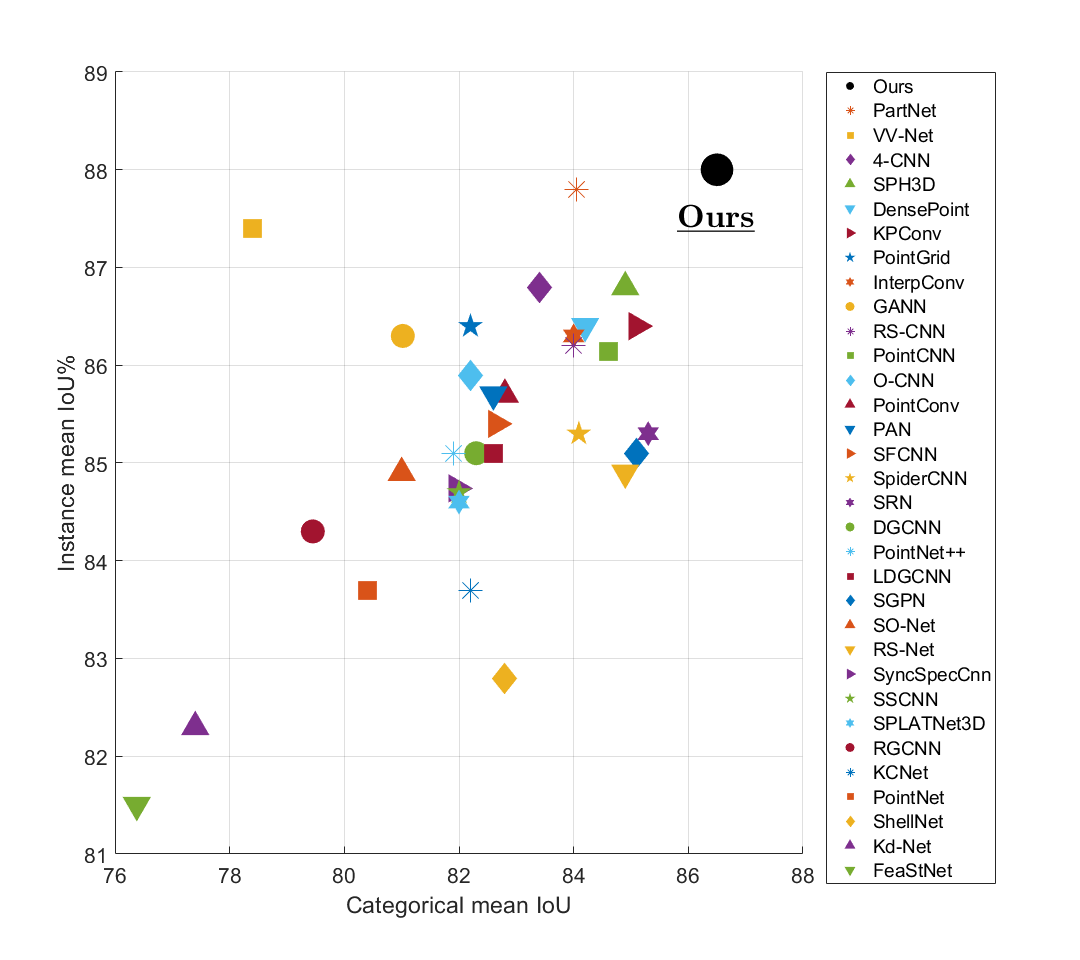}
	\vspace{-3mm}
	\caption{\footnotesize State-of-the-art part segmentation performance comparison on ShapeNet, where IoU denotes intersection-over-union.
	}
	\label{fig:ShapeNet}
	\vspace{-6mm}
\end{figure}

This observation inspires us to investigate whether in the literature there exists a state-of-the-art method that applies conventional 2D CNNs as backbone to image representations for 3D point cloud segmentation. Surprisingly, as we summarize in Table \ref{tab:component_comp}, we can only find a few, indicating that currently such integrated methods for point cloud segmentation may be significantly underestimated. Clearly the key challenge for developing such integrated methods is: 

\BlankLine
{\centering{\em How to effectively and efficiently project 3D point clouds into a 2D image space so that we can take advantage of local pattern extraction in conventional 2D CNNs for point cloud semantic segmentation?}\par}
\BlankLine

\bfsection{Approach}
The question above is nontrivial. A bad projection function can easily lead to the loss of structural information in a point cloud with, for instance, many point collisions in the image space. Such structural loss is fatal as it may introduce so much noise that the local patterns in the original cloud are completely changed, leading to poor performance even using 2D conventional CNNs. Therefore, a good point-to-image projection function is the key to bridge the gap between the point cloud inputs and 2D CNNs.

At the system level, our integrated method is as follows:
\setlist[enumerate]{leftmargin=12mm}
\begin{enumerate}[nosep]
    \item[{\bf Step 1.}] Construct graphs from point clouds.
    \item[{\bf Step 2.}] Project graphs into images using graph drawing.
    \item[{\bf Step 3.}] Segment points using U-Net.
\end{enumerate}

We are motivated by the graph visualization techniques in {\em graph drawing}, an area of mathematics and computer sciences whose goal is to present the nodes and edges of a graph on a plane with some specific properties \cite{chrobak1995linear,schnyder1990embedding,kamada1989algorithm,frishman2007multi}. Particularly the {\em Kamada-Kawai (KK)} algorithm \cite{kamada1989algorithm} is one of the most widely-used undirected graph visualization techniques. In general, the KK algorithm defines an objective function that measures the energy of each graph layout \wrt some graph distance, and searches for the (local) minimum that gives a reasonably good 2D visualization. Note that the KK algorithm works in a continuous 2D space, rather than 2D grid (\ie a discrete space).

Therefore, intuitively we propose an integer programming (IP) to enforce the KK algorithm to learn projections on 2D grid, leading to an NP-complete problem \cite{wolsey2014integer}. Considering that the computational complexity of the KK algorithm is at least $O(n^2)$ \cite{kobourov2012spring} with the number of nodes $n$ in a graph (\eg thousands of points in a cloud), it would be still too expensive to compute even if we relax the IP with rounding. 

In order to accelerate the computation in our approach, we follow the hierarchical strategy in \cite{fruchterman1991graph, meyerhenke2015drawing, jacomy2014forceatlas2} and further propose a novel hierarchical approximation with complexity of $O(n^{\frac{L+1}{L}})$, roughly speaking, where $L$ denotes the number of the levels in the hierarchy. In fact, such a hierarchical scheme can also help us reduce the complexity in graph construction from point clouds using Delaunay triangulation \cite{delaunay1934sphere} with worst-case complexity of $O(n^2)$ for 3D points \cite{amenta2007complexity}.

Once we learn the graph-to-grid projection for a point cloud, we accordingly generate an $(x,y,z)$-image by filling it in with 3D points and zeros. We further feed these image representations to a multi-scale U-Net \cite{ronneberger2015u} for segmentation.

\bfsection{Performance Preview}
To demonstrate how well our approach works, we summarize 32 state-of-the-art performance on a benchmark data set, ShapeNet \cite{yi2016ShapeNet}, in Fig. \ref{fig:ShapeNet} and compare ours with these results under the same training/testing protocols. Clearly our results are significantly better than all the others with large margins. Similar observations have been made on PartNet \cite{Yu_2019_PartNet} as well. Please refer to our experimental section for more details. 

\bfsection{Contributions}
In summary, our key contributions in this paper are as follows:
\setlist[itemize]{leftmargin=*}
\begin{itemize}[nosep]
\item We are the {\em first}, to the best of our knowledge, to explore the graph drawing algorithms in the context of learning 2D image representations for 3D point cloud segmentation.

\item We accordingly propose a novel hierarchical approximate algorithm that accounts for computation to map point clouds into image representations as well as preserving the local information among the points in each cloud.

\item We demonstrate the state-of-the-art performance on both ShapeNet and PartNet with significant improvement over the literature for 3D point cloud segmentation, using the integrated method of our graph drawing algorithm with the Delaunay triangulation and a multi-scale U-Net.
\end{itemize}

\section{Related Work}
Table \ref{tab:component_comp} summarizes some existing works. In particular,



\bfsection{Representations of 3D Point Clouds}
Voxels are popular choices because they can benefit from the efficient CNNs. PointGrid \cite{le2018pointgrid}, O-CNN \cite{wang2017cnn}, VV-Net \cite{Meng_2019_VV_Net} and InterpConv \cite{Mao_2019_Interpolated} sample a point cloud in volumetric grids and apply 3D CNNs. Some other works represent a point cloud in specific 2D domains and perform customized network operators \cite{su2018splatnet, rao2019spherical, yi2017syncspeccnn}. However, these works have difficulty in sampling from a non-uniformly distributed point cloud and result in a serious problem of point collisions. Graph-based approaches \cite{te2018rgcnn,Jiang_2019_Hierarchical,xu2018spidercnn,pan2019pointatrousnet,li2019graph,wang2019dynamic,shen2018mining,liang2019hierarchical,Wang_2019_Graph,landrieu2018large} construct graphs from point clouds for network processing by sampling all points as graph vertices. However, they often struggle in assigning edges between the graph vertices. There also exist some works \cite{qi2017pointnet,Pham_2019_JSIS3D,hua2018pointwise,qi2017pointnet++,Duan_2019_SRN,wang2018sgpn,Liu_2019_RS_CNN,Komarichev_2019_A_CNN,Thomas_2019_KPConv,li2018so_net,Wu_2019_PointConv,li2018pointcnn,Zhao_2019_PointWeb,Zhang2019ShellNet,Liu_2019_DensePoint,Yu_2019_PartNet} that directly use points as network inputs. Though they do not have to consider the sampling and local connections among the points, 
significant effort has been made to hierarchically partition and extract features from the local point sets. 
FoldingNet \cite{yang2018foldingnet} introduced 2D grid as a latent space, rather than the output space, to capture the geometry of a point cloud.

There are some works in the literature of light detection and ranging (LiDAR) point processing that utilize depth images \cite{silberman2012indoor} or $(x,y,z)$-images \cite{lyu2018real, lyu2018chipnet, milioto2019rangenet++,wu2017squeezeseg,wu2019squeezesegv2}  generated from LiDAR points for training networks. In Table~\ref{tab:component_comp} we summarize these works as well, even though they are not developed for point cloud segmentation.

\begin{table}[t]
	\caption{\footnotesize Summary of state-of-the-art methods for point segmentation.}
	\label{tab:component_comp}
	\vspace{-2mm}
	\adjustbox{width=1.0\columnwidth}{
		\centering
		\begin{tabular}{c|c|c|c|c}
			\toprule
			Method & Raw Data & \begin{tabular}[c]{@{}c@{}}Data-to-Input \\ Mapping\end{tabular}
			& \begin{tabular}[c]{@{}c@{}}Network \\Input\end{tabular}
			& \begin{tabular}[c]{@{}c@{}}Network \\ Architecture \end{tabular}
			\\
			\midrule
			\bf{Ours} & {\bf Point Cloud} & \bf{Graph Drawing} & \bf{(x,y,z)-Image} & \bf{Multi-scale U-Net} \\
			
			\midrule
			
			Lyu \etal \cite{lyu2018real} & LiDAR Frame & Sphere Mapping & (x,y,z,$\phi$,$\theta$,$\rho$,i)-Image & 2D CNN\\
			ChipNet \cite{lyu2018chipnet} & LiDAR Frame & Sphere Mapping & (x,y,z,$\phi$,$\theta$,$\rho$,i)-Image & 2D CNN\\
			LoDNN \cite{caltagirone2017fast} & LiDAR Frame & 2D Grid Samp. &  Statistics-Image & 2D CNN\\
			RangeNet++ \cite{milioto2019rangenet++} & LiDAR Frame & Sphere Mapping & (r,x,y,z,i)-Image & 2D CNN \\
			SqueezeSeg \cite{wu2017squeezeseg} & LiDAR Frame & Sphere Mapping & (x,y,z,i)-Image & 2D CNN \\
			SqueezeSegV2 \cite{wu2019squeezesegv2} & LiDAR Frame & Sphere Mapping & (x,y,z,i)-Image & 2D CNN \\			
			
			\midrule
			PointNet \cite{qi2017pointnet} & Point Cloud & - & Point & MLP\\
			JSIS3D \cite{Pham_2019_JSIS3D} & Point Cloud & - & Point & MT-PNet \\
			PCNN \cite{hua2018pointwise} & Point Cloud & - & Point & Pointwise Conv. \\
			PointNet ++ \cite{qi2017pointnet++} & Point Cloud & FPS & Point & MLP\\
			SRN \cite{Duan_2019_SRN} & Point Cloud & FPS & Point & SRN \\
			SGPN \cite{wang2018sgpn} & Point Cloud & FPS & Point & MLP \\
			RS-CNN \cite{Liu_2019_RS_CNN} & Point Cloud & Ball Query & Point & RS Conv. \\
			A-CNN \cite{Komarichev_2019_A_CNN} & Point Cloud & Ball Query & Point & A-CNN\\
			KPConv \cite{Thomas_2019_KPConv} & Point Cloud & Ball Query & Point & KPConv \\
			So-Net \cite{li2018so_net} & Point Cloud & KNN  & Point & SOM \\
			PointConv \cite{Wu_2019_PointConv} & Point Cloud & KNN & Point & PointConv \\
			PointCNN \cite{li2018pointcnn} & Point Cloud & KNN & Point  & $\mathcal{X}$-Conv\\
			PointWeb \cite{Zhao_2019_PointWeb} & Point Cloud & KNN &  Point & AFA \\
			ShellNet \cite{Zhang2019ShellNet} & Point Cloud & KNN & Point & ShellNet\\
			DensePoint \cite{Liu_2019_DensePoint} & Point Cloud & Rand. Samp. & Point & PConv. \\

			PartNet \cite{Yu_2019_PartNet} & Point Cloud & Latent Tree & Point & MLP\\
			Kd-Net \cite{klokov2017escape} & Point Cloud & Kd-Tree & Point & ConvNets \\
			MAP-VAE \cite{Han_2019_MAP_VAE} & Point Cloud & Latent Tree & Point & GRU \\
			
			\midrule
			
			RGCNN \cite{te2018rgcnn} & Point Cloud & Complete Graph & Graph & Graph Conv. \\
			Point-Edge \cite{Jiang_2019_Hierarchical} & Point Cloud & FPS & Graph & Point-Edge Net \\
			
			SpiderCNN \cite{xu2018spidercnn} & Point Cloud & KNN & Graph & SpiderConv \\
			PAN \cite{pan2019pointatrousnet} & Point Cloud & KNN & Graph & Point Atrous Conv. \\
			GANN \cite{li2019graph} & Point Cloud & KNN & Graph & Graph Attention \\
			DG-CNN \cite{wang2019dynamic} & Point Cloud & KNN & Graph &  Edge-Conv \\
			Kc-Net \cite{shen2018mining} & Point Cloud & KNN & Graph & MLP \\
			HDGCN \cite{liang2019hierarchical} & Point Cloud & KNN & Graph & MLP \\
			GAC-Net \cite{Wang_2019_Graph} & Point Cloud & Rand. Samp. & Graph & Graph Attention\\
			SPGraph \cite{landrieu2018large} & Point Cloud & Voronoi Adj. Graph & Graph & GRU\\
			\midrule
			RS-Net \cite{huang2018recurrent} & Point Cloud & 3D Grid Samp. & Voxel & RNN \\
			
			VV-Net \cite{Meng_2019_VV_Net} & Point Cloud & 3D Grid Samp.& Voxel & VAE \\
			PointGrid \cite{le2018pointgrid} & Point Cloud & 3D Grid Samp.& Voxel & 3D CNN \\
			InterpConv \cite{Mao_2019_Interpolated} & Point Cloud & 3D Grid Interpolate & Voxel & InterpConv\\
			O-CNN \cite{wang2017cnn} & Point Cloud & Octree & Voxel & MLP \\
			$\Psi$-CNN \cite{Lei_2019_Octree} & Point Cloud & Spherical 3D Grid & Voxel & Spherical Conv.\\
			SPLATNet \cite{su2018splatnet} & Point Cloud & Lattice Interpolate & Lattice & Bilateral Conv.\\
			SFCNN \cite{rao2019spherical} & Point Cloud & Sphere Mapping  & Sphere & Sph. Fractal Conv.\\
			SyncSpecCNN \cite{yi2017syncspeccnn} & Point Cloud & 3D Grid Samp. & Spectral & SpecTN \\

			\bottomrule
		\end{tabular}
	}
	\vspace{-3mm}
\end{table}

In contrast, we propose an efficient hierarchical approximate algorithm with Delaunay triangulation to map each point cloud onto a 2D image space.

\bfsection{Network Architectures}
Network operations are the key to hierarchically learn the local context and perform semantic segmentation on point clouds. Grid-based approaches usually apply regular 2D or 3D CNNs on the grid representations. Graph-based approaches usually apply customized convolutions on graph representations. For point-based approaches, MLP is the most widely used network. For some other point-based approaches, customized convolution operators are designed as well to support their own network architectures. Recurrent neural networks (RNNs) are applied in some works to handle the unfixed-sized point inputs. Please refer to Table \ref{tab:component_comp} for more references.

In contrast, we apply the classic U-Net to our image representations for point cloud segmentation. In our ablation study later, we also test several alternative 2D CNN architectures and all of them get comparable results to the literature.

\bfsection{Graph Drawing}
According to the purposes of graph layout, there exist two families of graph drawing algorithms, in general. N-planar graph \cite{schnyder1990embedding} focuses on presenting the graph on a plane with least edge intersections regardless the implicit topological features. Force-directed approaches such the KK algorithm, on the other hand, focus on minimizing the difference of graph node adjacency before and after 2D layout. 
Fruchterman-Reingold (FR) \cite{fruchterman1991graph}, FM$^3$ \cite{meyerhenke2015drawing} and ForceAtlas2 \cite{jacomy2014forceatlas2} speed up the force-directed layout computation for large-scale graphs by introducing hierarchical schemes and optimized iterating functions. 

Note that graph drawing can be considered as a subdiscipline of network embedding \cite{hamilton2017representation, cui2018survey, cai2018comprehensive} whose goal is to find a low dimensional representation of the network nodes in some metric space so that the given similarity (or distance) function is preserved as much as possible. In summary, graph drawing focuses on the 2D/3D visualization of graphs, while network embedding emphasizes the learning of low dimensional graph representations.

In this paper, we propose a hierachical graph drawing algorithm based on the KK algorithm, where we apply the FR method as layout initialization and then apply a novel discretization method to achieve grid layout.

\section{Our Method: A System Overview} \label{sec:Method}

\subsection{Graph Construction from Point Clouds}
In the literature a graph from a point cloud is usually generated by connecting the $K$ nearest neighbours (KNN) of each point. However, such KNN approaches suffer from selecting a suitable $K$. When $K$ is too small, the points are intended to form small subgraphs (\ie clusters) with no guarantee of connectivity among the subgraphs. When $K$ is too large, points are densely connected, leading to much more noise in local feature extraction.

In contrast, in this work we employ the Delaunary triangulation \cite{delaunay1934sphere}, a widely-used triangulation method in computational geometry, to create graphs based on the positions of points. The triangulation graph has three advantages: (1) The connection of all the nodes in the graph is guaranteed; (2) All the local nodes are directly connected; (3) The total number of graph connections is relatively small. In our experiments we found that the Delaunary triangulation gives us slightly better segmentation performance than the best one using KNN $(K=20)$ with margin of about 0.7\%.

The worst-case computational complexity of the Delaunay triangulation is $O(n^{\lceil\frac{d}{2}\rceil})$ \cite{amenta2007complexity} where $d$ is the feature dimension and $\lceil\cdot\rceil$ denotes the ceiling operation. Thus in the 3D space the complexity is $O(n^2)$ with $d=3$. In our experiments we found that its running time on 2048 points is about 0.1s (CPU:Intel Xeon E5-2687W@3.1GHz), on average. 

\subsection{Graph Drawing: from Graphs to Images}
Let $\mathcal{G}=(\mathcal{V}, \mathcal{E})$ be an undirected graph with a vertex set $\mathcal{V}$ and an edge set $\mathcal{E}\subseteq\mathcal{V}\times\mathcal{V}$. $s_{ij}\geq1, \forall i\neq j$ is the graph-theoretic distance such as shortest-path between two vertices $v_i, v_j\in\mathcal{V}$ on the graph that encodes the graph topology. 

Now we would like to learn a function $f:\mathcal{V}\rightarrow\mathbb{Z}^2$ to map the graph vertex set to a set of 2D integer coordinates on the grid so that the graph topology can be preserved as much as possible given a metric $d:\mathbb{R}^2\times\mathbb{R}^2\rightarrow\mathbb{R}$ and a loss $\ell:\mathbb{R}\times\mathbb{R}\rightarrow\mathbb{R}$. As a result, we are seeking for $f$ to minimize the objective $\min_{f}\sum_{i\neq j}\ell(d(f(v_i), f(v_j)), s_{ij})$. Letting $\mathbf{x}_i=f(v_i)\in\mathbb{Z}^2$ as reparametrization, we can rewrite this objective as an {\em integer programming} (IP) problem $\min_{\mathcal{X}\subseteq\mathbb{Z}^2}\sum_{i\neq j}\ell(d(\mathbf{x}_i, \mathbf{x}_j), s_{ij})$, 
where the set $\mathcal{X}=\{\mathbf{x}_i\}$ denotes the {\em 2D grid layout} of the graph, \ie all the vertex coordinates on the 2D grid.

For simplicity we set $\ell$ and $d$ above to the least-square loss and Euclidean distance to preserve topology, respectively. This leads us to the following objective for learning:
\begin{align}\label{eqn:kk}
\hspace{-3mm} \min_{\mathcal{X}\subseteq\mathbb{Z}^2}\sum_{i\neq j}\frac{1}{2}\left(\frac{\|\mathbf{x}_i - \mathbf{x}_j\|}{s_{ij}} - 1\right)^2, \; \mbox{s.t.} \; \mathbf{x}_i\neq\mathbf{x}_j, \forall i\neq j.
\end{align}

In fact the KK algorithm shares the same objective as in Eq. \ref{eqn:kk}, but with different feasible solution space in $\mathbb{R}^2$, leading to relatively faster solutions that are used as the initialization in our algorithm later (see Alg. \ref{alg:LayoutDiscretization}). Once the location of a point on the grid is determined, we associate its 3D feature as well as label, if available, with the location, finally leading to the $(x,y,z)$-image representation and a label mask with the same image size for network training.

In general an IP problem is NP-complete \cite{wolsey2014integer} and thus finding exact solutions is challenging. Relaxation and rounding is a widely used heuristic for solving IP due to its efficiency \cite{bradley1977applied}, where rounding is applied to the solution from the relaxed problem as the solution for the IP problem. However, considering that the computational complexity of the KK algorithm is at least $O(n^2)$ \cite{kobourov2012spring} with the number of nodes $n$ in a graph (\ie thousands of points in a cloud for our case), it would be still too expensive to compute even if we relax the IP with rounding. Empirically we found that the running time of the KK algorithm on 2048 points is about 38s (CPU:Intel Xeon E5-2687W@3.1GHz), on average, which is considerably long. To accelerate the computation in practice, we propose a novel hierarchical solution in Sec. \ref{sec:hierarchy}. 

\subsection{Multi-Scale U-Net for Point Segmentation}
\begin{figure}[t]
	\centering
	\includegraphics[width=\columnwidth]{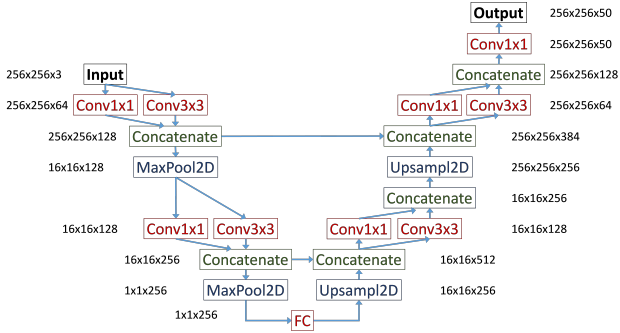}
	\vspace{-7mm}
	\caption{\footnotesize Illustration of our multi-scale U-Net architecture.}
	\label{fig:unet}
	\vspace{-3mm}
\end{figure}

Eq. \ref{eqn:kk} enforces our image representations for the point clouds to be compact, indicating that the local structures in a point cloud are very likely to be preserved as local patches in its image representation. This is crucial for 2D CNNs to work because as such small convolutional kernels (\eg $3\times3$) can be used for local feature extraction. 

To capture these local patterns in images, multi-scale convolutions are often used in networks such as the inception module in GoogLeNet \cite{szegedy2015going}. U-Net \cite{ronneberger2015u} was proposed for biomedical image segmentation, and its variants are widely used for different image segmentation tasks. As illustrated in Fig. \ref{fig:unet}, in this paper we propose a multi-scale U-Net that integrates the inception module with U-Net, where FC stands for the fully connected layer, ReLU activation is applied after each Inception module and FC layer, and the softmax activation is applied after the last Conv$1\times1$ layer. 

\setlength{\columnsep}{10pt}
\begin{wraptable}{r}{.55\columnwidth}\footnotesize
\vspace{-3mm}
\caption{\footnotesize Performance comparison on ShapeNet using different U-Nets.}\label{tab:unet}
\vspace{-3mm}
\setlength{\tabcolsep}{1.5pt}
    \centering
		\begin{tabular}{c|c|c|c}
			\toprule
			Scales in U-Net & 1x1 & 3x3  & Inception \\
			\midrule
			Instance mIoU (\%) & 83.1 & 82.5 & 88.0 \\
			\bottomrule
		\end{tabular}
\vspace{-3mm}
\end{wraptable} 
\bfsection{Single-Scale \vs Multi-Scale}
We only consider two sizes of 2D convolution kernels, \ie $1\times1$ and $3\times3$, because in our experiments we found that larger sizes of kernels do not bring significant improvement but heavier computational burden. We also compare the performance using single \vs multiple scales in Table \ref{tab:unet}. As we see the multi-scale U-Net with the inception module significantly outperforms the other single scale U-Nets.

\begin{figure}[t]
	\centering
	\includegraphics[width=\columnwidth]{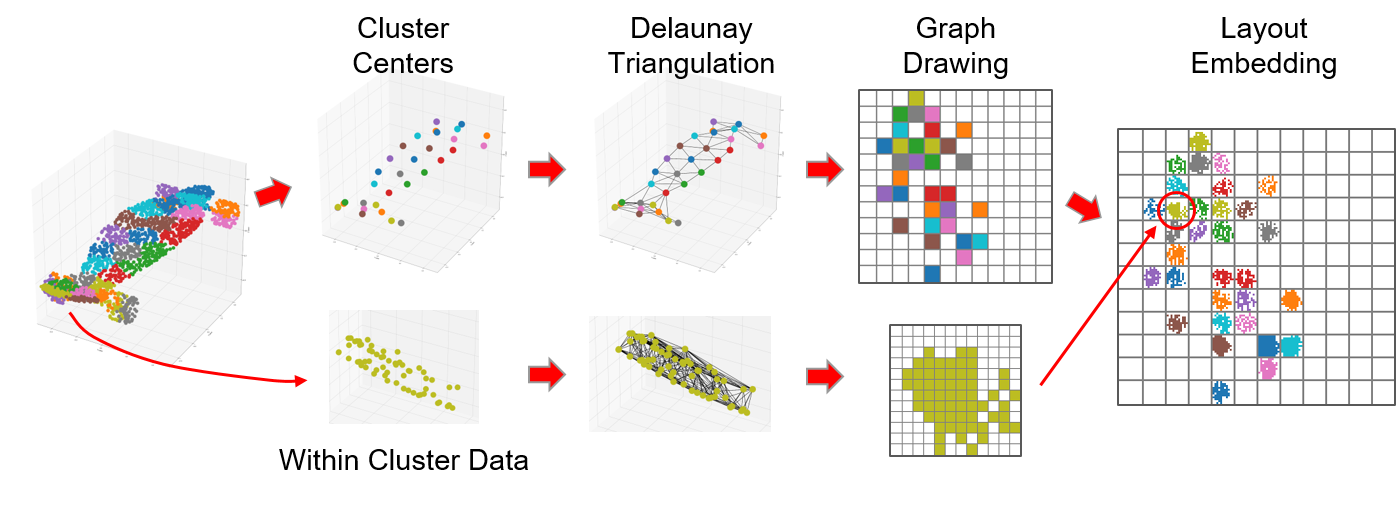}
	\vspace{-5mm}
	\caption{\footnotesize Illustration of hierarchical approximation for a point cloud. Each color represents a cluster where all the points share the same color.}
	\label{fig:hierarchical_graph_drawing}
	\vspace{-5mm}
\end{figure}

\setlength{\columnsep}{10pt}
\begin{wraptable}{r}{.67\columnwidth}\footnotesize
\vspace{-3mm}
\caption{\footnotesize Instance mIoU comparison on ShapeNet using different CNNs.}\label{tab:unet_cnn}
\vspace{-3mm}
\setlength{\tabcolsep}{1.2pt}
    \centering
		\begin{tabular}{c|c|c|c|c}
			\toprule
			CNNs & Conv1x1 & Conv3x3  & SegNet \cite{badrinarayanan2017segnet} & U-Net \\
			\midrule
			mIoU (\%) & 81.6 & 78.1 & 86.9 & 88.0 \\
			\bottomrule
		\end{tabular}
\vspace{-3mm}
\end{wraptable} 
\bfsection{U-Net \vs CNNs}
We also compare our U-Net with some other CNN architectures in Table \ref{tab:unet_cnn}. A baseline is an autoencoder-decoder network with similar architecture in Fig. \ref{fig:unet} but no multi-scales and skip connections. We test it with $1\times1$ and $3\times3$ kernels, respectively, as shown in Table \ref{tab:unet_cnn}. A second baseline is SegNet \cite{badrinarayanan2017segnet}, a much more complicated autoencoder-decoder. Again our U-Net works the best. By comparing Table \ref{tab:unet_cnn} and Table \ref{tab:unet}, we can see that the skip connections in U-Net really help improve the performance. Note that our simple baselines can achieve comparable performance with the literature already.

All the comparisons above are based on the same image representations under the same protocols. Please refer to our experimental section for more details.


\section{Efficient Hierarchical Approximation}\label{sec:hierarchy}

\subsection{Two-Level Graph Drawing}
For simplicity, in this section we will use the example in Fig. \ref{fig:hierarchical_graph_drawing} to explain the key components in our hierarchical approximation. All the operations here can be easily extended to hierarchical cases with no change.

Given a point cloud, we first cluster these points hierarchically. We then apply the Delaunay triangulation and our graph drawing algorithms sequentially to the cluster centers as well as the within-cluster points per cluster, respectively, producing higher and lower-level graph layouts. Finally we embed all the lower-level graph layouts into the higher-level layout (recursively along the hierarchy) to produce the 2D image representation. For instance, we cluster a 2048-point cloud from ShapeNet into 32 clusters, and build a higher-level grid with size $16\times16$ using these 32 cluster centers. Within each cluster we build a lower-level grid with size $16\times16$ as well using the points belonging to the cluster. We finally construct the image representation for the cloud with size $256\times256$.

\begin{algorithm}[t]\footnotesize
	\SetAlgoLined
	\SetKwInOut{Input}{Input}\SetKwInOut{Output}{Output}
    \Input{point cloud $\mathcal{P} = \{\mathbf{p}\}$, number of clusters $K$, parameter $\alpha$, \\distance metric $s$, cluster center computing function $c$}
	\Output{balanced point clusters $\mathcal{H}$}
	\BlankLine
	$\mathcal{H}\leftarrow \mbox{KMeans}(\mathcal{P}, K)$;
	
    \While{$\exists h^* \in \mathcal{H}, |h^*|>\alpha\frac{ |\mathcal{P}|}{K} $}
        {
        
        $h' \in\argmin_{\Tilde{h}: |\Tilde{h}|<
        \frac{|\mathcal{P}|}{K}} \left\{s(c(h^*),c(\Tilde{h}))\right\}$;
        
        $\mathbf{p}' \in \argmin_{\mathbf{p}\in h^*}\left\{s(\mathbf{p}, c(h'))\right\}$;
        
        $h^* \leftarrow h^*\setminus \{\mathbf{p}'\}; \; h' \leftarrow h' \cup \{\mathbf{p}'\}$;
        
        }
	\caption{\footnotesize Balanced KMeans for Clustering}\label{alg:Cluster_Balancing}
\end{algorithm}
\begin{algorithm}[t]\footnotesize
	\SetAlgoLined
	\SetKwInOut{Input}{Input}\SetKwInOut{Output}{Output}
    \Input{Graph $\mathcal{G}$, 2D grid $\mathcal{S}\subseteq\mathbb{Z}^2$}
	\Output{Graph layout $\mathcal{X}\subseteq\mathbb{Z}^2$}
	\BlankLine
	$\mathcal{X} \leftarrow \mbox{KK\_2D\_layout}(\mathcal{G}); \mathbf{a}\leftarrow\mbox{mean}(\mathcal{X}); \mathbf{b}\leftarrow\mbox{std}(\mathcal{X})$;
	
	\lForEach{$\mathbf{x}\in\mathcal{X}$}{
	    $\mathbf{x}\leftarrow\mbox{round}((\mathbf{x}-\mathbf{a})./\mathbf{b}*\sqrt{|\mathcal{X}|})$
	}
	
	\While{$\exists \mathbf{x}_i = \mathbf{x}_j, i\neq j, \mathbf{x}_i\in \mathcal{X},\mathbf{x}_j \in \mathcal{X}$}{
        
        $\mathbf{x}^*\in\argmin_{\mathbf{x}\in\mathcal{S}\setminus\mathcal{X}}\|\mathbf{x}_i-\mathbf{x}\|; \mathbf{x}_i\leftarrow\mathbf{x}^*$;
        
        
	}
	
	return $\mathcal{X}$;
	
	\caption{\footnotesize Fast Graph-to-Image Drawing Algorithm }\label{alg:LayoutDiscretization}
	
\end{algorithm}

\begin{figure*}[t]
	\centering
	\includegraphics[width=0.9\textwidth]{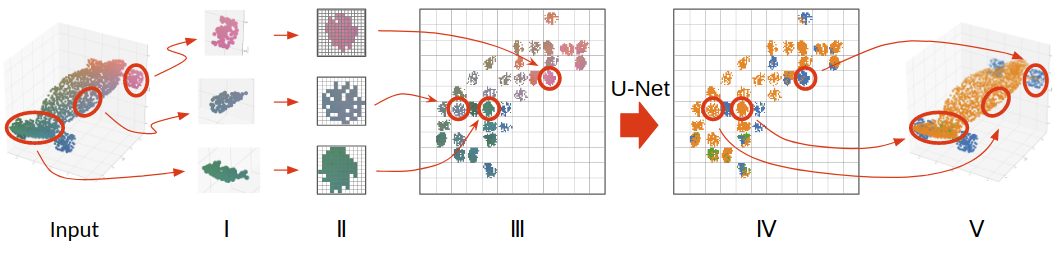}
	\vspace{-5mm}
	\caption{\footnotesize Illustration of our pipeline for point cloud semantic segmentation. Input: point cloud of a skateboard from ShapeNet. (\rom{1}): point cloud clustering, (\rom{2}): within-cluster image representation from graph drawing,
	(\rom{3}): image embedding to generate a representation for the cloud,
	(\rom{4}): image segmentation using U-Net,
	(\rom{5}): prediction reversion from the image representation to the point cloud. Here colors indicate either $(x,y,z)$ features or the predicted labels.
	}
	\label{fig:Pipeline}
	\vspace{-5mm}
\end{figure*}

\subsubsection{Balanced KMeans for Clustering}

The key to accelerate computation in graph construction from point clouds is to reduce the number of points that the triangulation and graph drawing algorithms process at a time. Therefore, without loss of information we introduce hierarchical clustering, following the strategy in \cite{fruchterman1991graph, meyerhenke2015drawing, jacomy2014forceatlas2}.

Recall that the complexity of the Delaunay triangulation and KK algorithms is $O(n^2)$, roughly speaking. Now consider the problem where given $n$ points how we should determine $K$ clusters so that the complexity in our graph construction from point clouds is minimize. The solution of this problem is that, ideally, all the clusters should have equal size of $\frac{n}{K}$, \ie balancing. Some algorithms such as normalized cut \cite{shi2000normalized} are developed for learning balanced clusters, however, suffering from high complexity. Fast algorithms such as KMeans, unfortunately, do not provide such balanced clusters by nature.

We thus propose a heuristic post-processing step on top of KMeans to approximately balance the clusters with condition $|h|\leq\alpha\frac{|\mathcal{P}|}{K}, \forall h\in\mathcal{H}$ where $\mathcal{P}=\{\mathbf{p}\}$ denotes a point cloud with size $|\mathcal{P}|$, $\mathcal{H}=\{h\}$ denotes a set of clusters (\ie point sets) with size $K$, $|h|$ denotes the size of cluster $h$, and $\alpha\geq1$ is a predefined constant. We list our algorithm in Alg.~\ref{alg:Cluster_Balancing}.

We first apply Kmeans to generate the cluster initials. We then target on one of the oversized clusters, $h^*$, at each iteration and change the cluster association for only one point. We determine the target cluster $h'$ as the closest not-full cluster to $h*$ to receive a point. To send a point from $h^*$ to $h'$, the selected point is a boundary point that is closest to the center of $h'$. By default we set $\alpha=1.2$, although we observed that higher values has little impact on either running time or performance. 


\subsubsection{Fast Graph-to-Image Drawing Algorithm}

Recall that our graph drawing algorithm in Eq. \ref{eqn:kk} is an IP problem with complexity of NP-complete. Even though we use hierarchical clustering to reduce the number of points for processing, solving the exact problem is still challenging. To overcome this problem, we propose a fast approximate algorithm in Alg. \ref{alg:LayoutDiscretization}, where $|\mathcal{X}|$ denotes the number of points.

\bfsection{Layout Discretization}
After the layout initialization with the KK algorithm, we discrete the layouts onto the 2D grid. We first normalize the layout to a Gaussian distribution with a zero mean and an identity standard deviation (std). Then we rescale each 2D point in the layout with a scaling factor $\sqrt{|\mathcal{X}|}$, followed by a rounding operator. The intuition behind this is to organize the layout within a $\sqrt{|\mathcal{X}|}\times\sqrt{|\mathcal{X}|}$ patch as tightly as possible while minimizing the topological change. We finally replace each collided point with its nearest empty cell on the grid sequentially as our final graph layout.

\bfsection{Point Collision}
In order to control the running time and image size in practice, we make a trade-off to predefine the maximum number of iterations as well as the maximum size of the 2D grid in Alg. \ref{alg:LayoutDiscretization}. This may incur that some 3D points will collide at the same location on the grid. Such point collision scenarios, however, are very rare in our experiments. For instance, using our implementation for ShapeNet we observe 26 collisions with $2\times26=52$ points (\ie 2 points per collision) among 5,885,952 points in the testing set when projected onto 2D grid, leading to a $8.8\times10^{-6}$ point collision ratio.

Once point collision occurs, we randomly select a point from the collided points and put the selected point at the location with its 3D feature $(x,y,z)$ and label, if available, for training U-Net. We observe that max pooling or average pooling is not appropriate to be applied here, because the labels of collided points can vary, \eg points at the boundary of different parts, leading to confusion for training U-Net. 

At test time, we propagate the predicted label of the selected point to all its collided points. We observe only 4 out of 52 points mislabelled on ShapeNet due to point collision. 

\subsection{Generalization}

\begin{wrapfigure}{r}{.3\linewidth}
    \vspace{-30pt}
    \begin{minipage}{\linewidth}
        \centering
        \includegraphics[width=\linewidth]{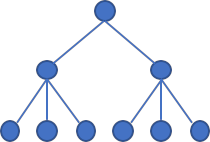}     
    \end{minipage}
    \vspace{-2mm}
    \caption{\footnotesize Full-tree illustration for our hierarchical clustering.
    }\label{fig:full_tree}
    \vspace{-5mm}
\end{wrapfigure}

Recall that we would like to achieve balanced clusters in our hierarchical method for computational efficiency. Therefore, as generalization we propose using the full tree data structure, as illustrated in Fig.~\ref{fig:full_tree}, to organize the hierarchical clusters, where at each cluster a higher-level graph is built using the Delaunay triangulation on the cluster centers, following by graph drawing to generate an image patch. Then we embed all the patches hierarchically to produce an image representation for a point cloud, and apply the remaining steps in Fig.~\ref{fig:Pipeline} for segmentation.

\bfsection{Complexity}
For simplicity and without loss of generality, assume that the full tree has $L\geq1$ levels, and each cluster at the same level contains the same number of points. Let $a_i, b_i$ be the numbers of clusters and sub-clusters per cluster at the $i$-th level, respectively, and $n$ be the total number of points. For instance, in Fig. \ref{fig:full_tree} we have $L=3, a_1=1, b_1=2, a_2=2, b_2=3, a_3=6, b_3=1, n=6$. Then it holds that $\prod_{j=i}^{L}b_i=\frac{n}{a_i}, \forall i$. We observe that in practice the running time of our hierarchical approximation is dominated by the KK initialization in Alg. \ref{alg:LayoutDiscretization} (see Table \ref{tab:running_time} for more details).

\begin{prop}[Complexity of Hierarchical Approximation]\label{prop:1}
Given a full tree with $(a_i, b_i), \forall i\in[L]$ as above, the complexity of our hierarchical approximation is dominated by $O\left(n^{\frac{L+1}{L}}\right)$, at least.
\end{prop}
\begin{proof}
Here we focus on the complexity of the KK algorithm as it dominates the whole. Since for each cluster this complexity is $O(b_i^2)$, the total complexity of our approach is $O(\sum_{i=1}^La_i b_i^2)$. Because
\begin{align}
    \sum_{i=1}^La_i b_i^2 & = n\sum_{i=1}^L\frac{b_i}{\prod_{j=i+1}^L b_j}\geq nL\left[\prod_i\left(\frac{b_i}{\prod_{j=i+1}^L b_j}\right)\right]^{\frac{1}{L}} \nonumber\\
    & = nL\left(\frac{n}{\prod_{i=2}^L b_i^{i-1}}\right)^{\frac{1}{L}} = O\left(n^{\frac{L+1}{L}}\right), 
\end{align}
we can complete the proof accordingly.
\end{proof}


\section{Experiments}
We evaluate our works on two benchmark data sets for point cloud segmentation: ShapeNet \cite{yi2016ShapeNet} and PartNet \cite{Mo_2019_PartNet_dataset}. We follow exactly the same experimental setups as in PointNet \cite{qi2017pointnet} for ShapeNet and \cite{Mo_2019_PartNet_dataset} for PartNet, respectively.

ShapeNet contains 16,881 CAD shape models (14,007 and 2,874 for training and testing, respectively) from 16 categories with 50 part categories. From each shape model 2048 points are scanned and labeled with their part categories. The shapes come from the same object category share the same part label sets, while shapes from different object categories have no shared part category. 
For performance evaluation there are two mean intersection-over-union (mIoU) metrics, namely, class mIoU and instance mIoU. Class mIoU is the average over points in each shape category, while instance mIoU is the average over all shape instances.

PartNet is a semantic segmentation benchmark focusing on fine-grained part-level 3D object understanding. Compared with ShapeNet, it has 24 shape categories and 26,671 shape instances. In addition, PartNet samples 10,000 points from each shape instance and defines up to 82 part semantics in one shape category, which calls for better local context learning to recognize them. Different from training a single network for all shape categories as done in ShapeNet, PartNet defines three segmentation levels in each shape category where a network is trained and tested for each category at each level separately.

\subsection{Our Pipeline for Point Cloud Segmentation}
In all of our experiments, we utilize the pipeline as illustrated in Fig. \ref{fig:Pipeline} for point cloud segmentation. As we expect, the 3D points are mapped to the 2D image space smoothly following their relative distances in the 3D space, leading to similar distributions in the neighborhood among image pixels to those in the local regions of the point cloud.

\bfsection{Implementation}
We use the KMeans solver in the Scikit-learn library \cite{Buitinck2013sklearn} with 100 iterations in maximum, the Delaunay triangulation implementation in the Scipy library \cite{Virtanen2019SciPy}, and the spring-layout implementation for graph drawing in the Networkx library \cite{Hagberg2008networkx}. In the mask image, we ignore the pixels with no point association that do not contribute to the loss of network training.

In our pipeline there are there hyper-parameters: the number of clusters $K$, the maximum ratio $\alpha$, and the grid sizes for both lower and higher-level graph drawing. By default, we set $K=32$ and $K=100$ on ShapeNet and PartNet, respectively. On both data sets we use $\alpha = 1.2$, and set both lower and higher-level grid size to $16\times16$, leading to a $256\times256$ image representation per cloud.

We implement our multi-scale U-Net in Keras with Tensorflow backend on a desktop machine with an Intel Xeon E5-2687W@3.1GHz CPU and an NVidia RTX 2080Ti GPU. During training we follow PointNet \cite{qi2017pointnet} to rotate and jitter the shape models as input. We use the Adam \cite{kingma2014adam} optimizer and set the learning rate to 0.0001. We train the network for 100 epochs with single batch in each iteration.

\begin{table}[t]
    \centering
	\caption{\footnotesize Running time of each component in our pipeline on ShapeNet. CPU: Intel Xeon E5-2687W@3.1GHz, GPU: NVidia RTX 2080Ti}
	\label{tab:running_time}
	\vspace{-2mm}
		\centering
		\setlength{\tabcolsep}{3.3pt}\footnotesize
		\begin{tabular}{c|cccc|c}
			\toprule
			Component &  \begin{tabular}[c]{@{}c@{}}KMeans \\ Clustering\end{tabular}   & \begin{tabular}[c]{@{}c@{}}Delaunay \\ Tribulation\end{tabular}  & \begin{tabular}[c]{@{}c@{}}Graph \\ Drawing\end{tabular}   & \begin{tabular}[c]{@{}c@{}}Network \\ Inference\end{tabular}  & Total\\
			\midrule
			Time (ms) & 65.0$\pm$7.2 & 41.0$\pm$7.1 & 696.5$\pm$30.0 &  18.6$\pm$2.6 & 1054.8\\
			\midrule
			Device & CPU & CPU & CPU & GPU & - \\
			\bottomrule
		\end{tabular}
	\vspace{-5mm}
\end{table}

\bfsection{Running Time}
We also list the average running time of each component in our pipeline in Table \ref{tab:running_time} for analysis. Compared with the running time on 2048 points, both Delaunay triangulation and graph drawing algorithms are accelerated significantly (recall 0.1s and 38s, respectively). Still graph drawing dominates the overall running time. Further acceleration will be considered in our future work.

\subsection{State-of-the-art Performance Comparison}

\subsubsection{Ablation Study}
In this section we evaluate the effects of different factors on our segmentation performance using ShapeNet. We keep using the default parameters and components in our pipeline, unless we explicitly mention what to change accordingly.

\begin{wraptable}{r}{.63\columnwidth}\footnotesize
\vspace{-3mm}
\caption{\footnotesize Instance mIoU results (\%) under different settings for $s_{ij}$.}\label{tab:ablation_sij}
\vspace{-3mm}
    \centering
    \setlength{\tabcolsep}{0.5pt}\scriptsize
		\begin{tabular}{c|c|c|c}
			\toprule
			$s_{ij}$ & \begin{tabular}[c]{@{}c@{}}Triangulation +\\ Shortest Path \cite{Wang_2019_Graph} \end{tabular} & \begin{tabular}[c]{@{}c@{}} KNN (K=20) + \\ Shortest Path \end{tabular}  & \begin{tabular}[c]{@{}c@{}} 3D Distance \\ \cite{te2018rgcnn} \end{tabular}   \\
			\midrule
			mIoU & 88.0 & 87.1 & 86.4\\
			\bottomrule
		\end{tabular}
\vspace{-2mm}
\end{wraptable} 

\bfsection{Graph Distance $s_{ij}$ in Eq. \ref{eqn:kk}}
There are multiple choices to compute $s_{ij}$ that our algorithm aims to preserve. We demonstrate three ways in Table \ref{tab:ablation_sij} to verify their effects on performance. Note that for the 3D distance method, we do not construct graphs from point cloud. Rather we directly compute the $(x,y,z)$-distance between pairs of points. As we see, different $s_{ij}$'s do have impact on our segmentation performance, but relatively small. Compared with the results in Fig. \ref{fig:ShapeNet}, even using the 3D distance our pipeline can still outperform all the competitors.

\begin{wraptable}{r}{.58\columnwidth}\footnotesize
\vspace{-3mm}
\caption{\footnotesize Result comparison with different sizes of image representation.}\label{tab:ablation_SIZE}
\vspace{-3mm}
    \centering
    \setlength{\tabcolsep}{0.5pt}\scriptsize
		\begin{tabular}{c|c|c|c}
			\toprule
			Grid Size & $10 \times 10$  & $16 \times 16$ & $24 \times 24$ \\
			\midrule 
			
			Image Size & $100 \times 100$ & $256 \times 256$ & $576 \times 576$\\
			\midrule
			mIoU (\%) & 82.4 & 88.0 & 87.5\\
			\midrule
			Time (ms) & 13.6 & 18.6 & 63.3\\
			\bottomrule
		\end{tabular}
\vspace{-2mm}
\end{wraptable} 
\bfsection{Grid Size in Graph Drawing}
Such grid size affects not only our segmentation performance but also the inference time of our pipeline. As demonstration we list three grid sizes in Table \ref{tab:ablation_SIZE}. As we expect, larger image sizes lead to significantly longer inference time, but marginal change in performance. Using smaller sizes it may be difficult to preserve the topological information among points, leading to performance degradation but faster inference time.

\begin{figure}[t]
	\centering
	\includegraphics[width=1.\columnwidth]{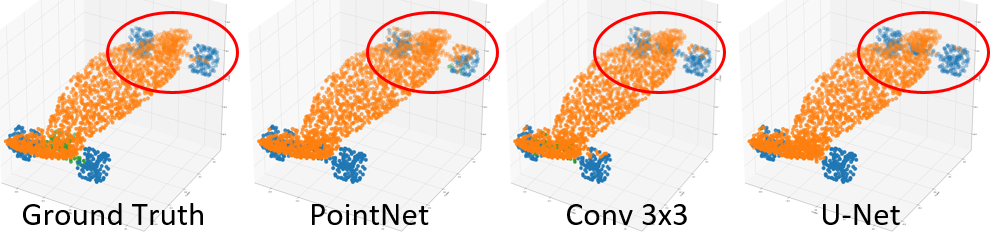}
	\vspace{-5mm}
	\caption{\footnotesize Visual comparison among different methods. {\bf Ours: U-Net}.}
	\label{fig:visual_comp}
	\vspace{-5mm}
\end{figure}

\begin{wraptable}{r}{.45\columnwidth}\footnotesize
\vspace{-3mm}
\caption{\footnotesize Result comparison with different numbers of clusters, $K$, in KMeans.}\label{tab:ablation_K}
\vspace{-3mm}
    \centering
    \setlength{\tabcolsep}{0.5pt}\footnotesize
		\begin{tabular}{c|c|c|c}
			\toprule
			$K$ & 32 & 64 & 128\\
			\midrule
			mIoU (\%) & 88.0 & 86.7 & 85.9\\
			\midrule
			Time (ms) & \;821.2\; & \;775.4\; & \;1164.5\;\\
			\bottomrule
		\end{tabular}
\vspace{-2mm}
\end{wraptable} 
\bfsection{Number of Clusters}
Similar to grid size, the number of clusters also has an impact on both segmentation performance and inference time. To verify this, we show a result comparison in Table \ref{tab:ablation_K}. With larger $K$, the performance decreases. This is probably because the higher-level graph loses more local context in the cloud so that even after pooling such loss cannot be recovered in learning. For timing, the numbers fluctuate due to different hierarchies as we prove in Prop.~\ref{prop:1}. 

\begin{table*}[t]
	\caption{\footnotesize Result comparison  (\%) with recent works on ShapeNet. Numbers in red are the best in the column, and numbers in blue are the second best.}
	\label{tab:Shapenet_Result_per_class}
	\vspace{-2mm}
		\centering
		\setlength{\tabcolsep}{3.3pt}\footnotesize
		\begin{tabular}{c|cc|cccccccccccccccc}
		    \toprule
            Method & 
            \begin{tabular}[c]{@{}c@{}}class \\ mIoU\end{tabular}
            &
            \begin{tabular}[c]{@{}c@{}}instance \\ mIoU\end{tabular}
            
            & \begin{tabular}[c]{@{}c@{}}air \\ plane \end{tabular} & bag & cap & car
            & chair & \begin{tabular}[c]{@{}c@{}}ear \\ phone\end{tabular} & guitar & knife
            & lamp  & laptop & \begin{tabular}[c]{@{}c@{}}motor \\ bike\end{tabular} & mug
            & pistol & rocket & \begin{tabular}[c]{@{}c@{}}skate \\ board\end{tabular} & table  \\
            \midrule 
            
            
            DGCNN \cite{wang2019dynamic}
            & 82.3 & 85.1
            & 84.2 & 83.7 & 84.4 & 77.1
            & 90.9 & 78.5 & 91.5 & 87.3
            & 82.9 & 96.0 & 67.8 & 93.3
            & 82.6 & 59.7 & 75.5 & 82.0 \\
            
            RS-CNN \cite{Liu_2019_RS_CNN}
            & 84.0 & 86.2
            & 83.5 & 84.8 & 88.0 & 79.6
            & 91.2 & 81.1 & 91.6 & 88.4
            & \textcolor{blue}{\bf{86.0}} & 96.0 & 73.7 & 94.1
            & 83.4 & 60.5 & 77.7 & 83.6 \\
            
            DensePoint \cite{Liu_2019_DensePoint}
            & 84.2 & 86.4
            & 84.0 & 85.4 & \textcolor{blue}{\bf{90.0}} & 79.2
            & 91.1 & 81.6 & 91.5 & 87.5
            & 84.7 & 95.9 & 74.3 & 94.6
            & 82.9 & 64.6 & 76.8 & 83.7 \\
            
            SpiderCNN \cite{xu2018spidercnn}
            & 84.1 & 85.3
            & 83.5 & 81.0 & 87.2 & 77.5
            & 90.7 & 76.8 & 91.1 & 87.3 
            & 83.3 & 95.8 & 70.2 & 93.5 
            & 82.7 & 59.7 & 75.8 & 82.8  \\       
            
            PointGrid \cite{le2018pointgrid}
            & 82.2 & 86.4
            & 85.7 & 82.5 & 81.8 & 77.9
            & \textcolor{red}{\bf{92.1}} & \textcolor{blue}{\bf{82.4}} & \textcolor{red}{\bf{92.7}} & 85.8 
            & 84.2 & 95.3 & 65.2 & 93.4 
            & 81.7 & 56.9 & 73.5 & 84.6  \\ 
            
            VV-Net \cite{Meng_2019_VV_Net}
            & 84.2 & \textcolor{blue}{\bf{87.4}}
            & 84.2 & \textcolor{red}{\bf{90.2}} & 72.4 & \textcolor{red}{\bf{83.9}}
            & 88.7 & 75.7 & \textcolor{blue}{\bf{92.6}} & 87.2
            & 79.8 & 94.9 & 73.4 & 94.4 
            & \textcolor{blue}{\bf{86.4}} & \textcolor{blue}{\bf{65.2}} & \textcolor{red}{\bf{87.2}} & \textcolor{red}{\bf{90.4}}  \\            
            
            PartNet \cite{Yu_2019_PartNet}
            & 84.1 & \textcolor{blue}{\bf{87.4}}
            & \textcolor{red}{\bf{87.8}} & 86.7 & 89.7 & 80.5
            & \textcolor{blue}{\bf{91.9}} & 75.7 & 91.8 & 85.9 
            & 83.6 & \textcolor{red}{\bf{97.0}} & 74.6 & \textcolor{red}{\bf{97.3}} 
            & 83.6 & 64.6 & 78.4 & 85.8  \\        
            
            $\Psi$-CNN \cite{Lei_2019_Octree}
            & 83.4 & 86.8
            & 84.2 & 82.1 & 83.8 & 80.5
            & 91.0 & 78.3 & 91.6 & 86.7 
            & 84.7 & 95.6 & \textcolor{blue}{\bf{74.8}} & 94.5
            & 83.4 & 61.3 & 75.9 & 85.9\\
            
            SFCNN \cite{rao2019spherical}
            & 82.7 & 85.4
            & 83.0 & 83.4 & 87.0 & 80.2
            & 90.1 & 75.9 & 91.1 & 86.2
            & 84.2 & \textcolor{blue}{\bf{96.7}} & 69.5 & 94.8
            & 82.5 & 59.9 & 75.1 & 82.9\\
            
            PAN \cite{pan2019pointatrousnet}
            & 82.6 & 85.7
            & 82.9 & 81.3 & 86.1 & 78.6
            & 91.0 & 77.9 & 90.9 & 87.3
            & 84.7 & 95.8 & 72.9 & 95.0
            & 80.8 & 59.6 & 74.1 & 83.5\\
            
            SRN \cite{Duan_2019_SRN}
            & 82.2 & 85.3
            & 82.4 & 79.8 & 88.1 & 77.9
            & 90.7 & 69.6 & 90.9 & 86.3
            & 84.0 & 95.4 & 72.2 & 94.9
            & 81.3 & 62.1 & 75.9 & 83.2\\

            PointCNN \cite{li2018pointcnn}
            & \textcolor{blue}{\bf{84.6}} & 86.1
            & 84.1 & 86.5 & 86.0 & \textcolor{blue}{\bf{80.8}}
            & 90.6 & 79.7 & 92.3 & \textcolor{blue}{\bf{88.4}}
            & 85.3 & 96.1 & \textcolor{red}{\bf{77.2}} & \textcolor{blue}{\bf{95.3}}
            & 84.2 & 64.2 & \textcolor{blue}{\bf{80.0}} & 83.0\\
            
            \midrule

            {\bf Ours}
            & \textcolor{red}{\bf{86.5}} & \textcolor{red}{\bf{88.0}}
            & \textcolor{blue}{\bf{86.8}} & \textcolor{blue}{\bf{87.5}} & \textcolor{red}{\bf{94.2}} & 77.5
            & 91.2 & \textcolor{red}{\bf{87.2}} & 91.0 & \textcolor{red}{\bf{90.4}}
            & \textcolor{red}{\bf{90.1}} & 96.3 & 72.9 & 93.0
            & \textcolor{red}{\bf{88.1}} & \textcolor{red}{\bf{70.6}} & \textcolor{blue}{\bf{80.0}} & \textcolor{blue}{\bf{86.6}} \\
			\bottomrule
		\end{tabular}
	\vspace{-2mm}
\end{table*}

\begin{table*}[t]
	\caption{\footnotesize Result comparison on PartNet using part-category mIoU (\%). P, P$^+$, S and C refer to PointNet \cite{qi2017pointnet}, PointNet++ \cite{qi2017pointnet++}, SpiderCNN \cite{xu2018spidercnn} and PointCNN \cite{li2018pointcnn}. 1, 2 and 3 refer to three tasks: coarse-, middle- and fine-grained. Short lines denote the undefined levels. Numbers are cited from \cite{Mo_2019_PartNet_dataset}. 
	}
	\label{tab:PartNet}
	\vspace{-2mm}
		\centering
		\setlength{\tabcolsep}{1.6pt}\footnotesize
		\begin{tabular}{c|c|cccccccccccccccccccccccc}
		    \toprule
             & Avg 
            & Bag & Bed & Bott & Bowl
            & Chair & Clock & Dish & Disp
            & Door & Ear & Fauc & Hat
            & Key & Knife & Lamp & Lap
            & Micro & Mug & Frid & Scis 
            & Stora & Table & Trash & Vase \\
            
            \midrule
            P1 & 57.9
            & 42.5 & 32.0 & 33.8 & 58.0 
            & 64.6 & 33.2 & 76.0 & 86.8 
            & 64.4 & 53.2 & 58.6 & 55.9
            & \textcolor{blue}{\bf{65.6}}  & 62.2 & 29.7 & 96.5
            & 49.4 & 80.0 & 49.6 & 86.4
            & 51.9 & 50.5 & 55.2 & 54.7\\
            
            P2 & 37.3
            & ---- & 20.1 & ---- & ---- 
            & 38.2 & ---- & 55.6 & ---- 
            & 38.3 & ---- & ---- & ---- 
            & ---- & ---- & 27.0 & ---- 
            & 41.7 & ---- & 35.5 & ---- 
            & 44.6 & 34.3 & ---- & ---- \\
            
            P3 & 35.6
            & ---- & 13.4 & 29.5 & ----
            & 27.8 & 28.4 & 48.9 & 76.5
            & 30.4 & 33.4 & 47.6 & ----
            & ---- & 32.9 & 18.9 & ----
            & 37.2 & ---- & 33.5 & ----
            & 38.0 & 29.0 & 34.8 & 44.4 \\
            \midrule
            
            Avg & 51.2 
            & 42.5 & 21.8 & 31.7 & 58.0 
            & 43.5 & 30.8 & 60.2 & 81.7 
            & 44.4 & 43.3 & 53.1 & 55.9 
            & \textcolor{blue}{\bf{65.6}} & 47.6 & 25.2 & 96.5 
            & 42.8 & 80.0 & 39.5 & 86.4 
            & 44.8 & 37.9 & 45.0 & 49.6\\
            \bottomrule
            
            P$^+$1 & \textcolor{blue}{\bf{65.5}}
            & \textcolor{blue}{\bf{59.7}} & 51.8 & 53.2 & 67.3
            & 68.0 & \textcolor{blue}{\bf{48.0}} & 80.6 & 89.7
            & 59.3 & \textcolor{blue}{\bf{68.5}} & 64.7 & 62.4
            & 62.2 & \textcolor{red}{\bf{64.9}} & \textcolor{blue}{\bf{39.0}} & \textcolor{blue}{\bf{96.6}}
            & 55.7 & 83.9 & 51.8 & \textcolor{blue}{\bf{87.4}}
            & 58.0 & \textcolor{blue}{\bf{69.5}} & 64.3 & \textcolor{blue}{\bf{64.4}}\\
            
            P$^+$2 & 44.5
            & ---- & 38.8 & ---- & ---- 
            & 43.6 & ---- & 55.3 & ---- 
            & 49.3 & ---- & ---- & ---- 
            & ---- & ---- & \textcolor{blue}{\bf{32.6}} & ---- 
            & 48.2 & ---- & 41.9 & ---- 
            & 49.6 & \textcolor{blue}{\bf{41.1}} & ---- & ----\\
            
            P$^+$3 & 42.5
            & ---- & 30.3 & 41.4 & ----
            & 39.2 & \textcolor{blue}{\bf{41.6}} & \textcolor{blue}{\bf{50.1}} & 80.7
            & 32.6 & 38.4 & \textcolor{blue}{\bf{52.4}} & ----
            & ---- & \textcolor{blue}{\bf{34.1}} & \textcolor{blue}{\bf{25.3}} & ----
            & 48.5 & ---- & 36.4 & ----
            & 40.5 & \textcolor{blue}{\bf{33.9}} & 46.7 & 49.8 \\
            
            \midrule
            Avg & 58.1 
            & \textcolor{blue}{\bf{59.7}} & 40.3 & 47.3 & 67.3
            & 50.3 & \textcolor{blue}{\bf{44.8}} & \textcolor{blue}{\bf{62.0}} & 85.2
            & 47.1 & 53.5 & 58.6 & 62.4
            & 62.2 & \textcolor{blue}{\bf{49.5}} & \textcolor{blue}{\bf{32.3}} & \textcolor{blue}{\bf{96.6}}
            & 50.8 & 83.9 & 43.4 & \textcolor{blue}{\bf{87.4}}
            & 49.4 & \textcolor{blue}{\bf{48.2}} & 55.5 & 57.1\\
            \bottomrule
            
            S1 & 60.4 
            & 57.2 & \textcolor{blue}{\bf{55.5}} & \textcolor{blue}{\bf{54.5}} & \textcolor{blue}{\bf{70.6}}
            & 67.4 & 33.3 & 70.4 & 90.6 
            & 52.6 & 46.2 & 59.8 & \textcolor{blue}{\bf{63.9}} 
            & 64.9 & 37.6 & 30.2 & \textcolor{red}{\bf{97.0}} 
            & 49.2 & \textcolor{blue}{\bf{83.6}} & 50.4 & 75.6 
            & \textcolor{blue}{\bf{61.9}} & 50.0 & 62.9 & 63.8\\
            
            S2 & 41.7 
            & ---- & \textcolor{blue}{\bf{40.8}} & ---- & ---- 
            & 39.6 & ---- & \textcolor{blue}{\bf{59.0}} & ---- 
            & 48.1 & ---- & ---- & ---- 
            & ---- & ---- & 24.9 & ---- 
            & 47.6 & ---- & 34.8 & ---- 
            & 46.0 & 34.5 & ---- & ---- \\
            
            S3 & 37.0 
            & ---- & 36.2 & 32.2 & ---- 
            & 30.0 & 24.8 & 50.0 & 80.1 
            & 30.5 & 37.2 & 44.1 & ---- 
            & ---- & 22.2 & 19.6 & ---- 
            & 43.9 & ---- & \textcolor{blue}{\bf{39.1}} & ---- 
            & \textcolor{blue}{\bf{44.6}} & 20.1 & 42.4 & 32.4\\
            \midrule
            Avg & 53.6 
            & 57.2 & 44.2 & 43.4 & \textcolor{blue}{\bf{70.6}}
            & 45.7 & 29.1 & 59.8 & 85.4 
            & 43.7 & 41.7 & 52.0 & \textcolor{blue}{\bf{63.9}} 
            & 64.9 & 29.9 & 24.9 & \textcolor{red}{\bf{97.0}} 
            & 46.9 & \textcolor{blue}{\bf{83.6}} & 41.4 & 75.6 
            & 50.8 & 34.9 & 52.7 & 48.1\\
            \bottomrule
            
            C1 & 64.3 
            & \textcolor{red}{\bf{66.5}} & \textcolor{red}{\bf{55.8}} & 49.7 & 61.7 
            & \textcolor{blue}{\bf{69.6}} & 42.7 & \textcolor{blue}{\bf{82.4}} & \textcolor{blue}{\bf{92.2}} 
            & \textcolor{blue}{\bf{63.3}} & 64.1 & \textcolor{blue}{\bf{68.7}} & \textcolor{red}{\bf{72.3}} 
            & \textcolor{red}{\bf{70.6}} & 62.6 & 21.3 & \textcolor{red}{\bf{97.0}} 
            & \textcolor{blue}{\bf{58.7}} & \textcolor{red}{\bf{86.5}} & \textcolor{blue}{\bf{55.2}} & \textcolor{red}{\bf{92.4}}
            & 61.4 & 17.3 & \textcolor{red}{\bf{66.8}} & 63.4\\
            
            C2 & \textcolor{blue}{\bf{46.5}} 
            & ---- & \textcolor{red}{\bf{42.6}} & ---- & ---- 
            & \textcolor{red}{\bf{47.4}} & ---- & \textcolor{red}{\bf{65.1}} & ---- 
            & \textcolor{blue}{\bf{49.4}} & ---- & ---- & ---- 
            & ---- & ---- & 22.9 & ---- 
            & \textcolor{blue}{\bf{62.2}} & ---- & \textcolor{blue}{\bf{42.6}} & ---- 
            & \textcolor{blue}{\bf{57.2}} & 29.1 & ---- & ----\\
            
            C3 & \textcolor{blue}{\bf{46.4}} 
            & ---- & \textcolor{red}{\bf{41.9}} & \textcolor{blue}{\bf{41.8}} & ---- 
            & \textcolor{red}{\bf{43.9}} & 36.3 & \textcolor{red}{\bf{58.7}} & \textcolor{blue}{\bf{82.5}} 
            & \textcolor{blue}{\bf{37.8}} & \textcolor{red}{\bf{48.9}} & \textcolor{red}{\bf{60.5}} & ---- 
            & ---- & \textcolor{blue}{\bf{34.1}} & 20.1 & ---- 
            & \textcolor{blue}{\bf{58.2}} & ---- & \textcolor{red}{\bf{42.9}} & ---- 
            & \textcolor{red}{\bf{49.4}} & 21.3 & \textcolor{blue}{\bf{53.1}} & \textcolor{blue}{\bf{58.9}}\\
            
            \midrule
            Avg & \textcolor{blue}{\bf{59.8}} 
            & \textcolor{red}{\bf{66.5}} & \textcolor{red}{\bf{46.8}} & \textcolor{blue}{\bf{45.8}} & 61.7 
            & \textcolor{blue}{\bf{53.6}} & 39.5 & \textcolor{red}{\bf{68.7}} & \textcolor{blue}{\bf{87.4}} 
            & \textcolor{blue}{\bf{50.2}} & \textcolor{blue}{\bf{56.5}} & \textcolor{red}{\bf{64.6}} & \textcolor{red}{\bf{72.3}} 
            & \textcolor{red}{\bf{70.6}} & 48.4 & 21.4 & \textcolor{red}{\bf{97.0}} 
            & \textcolor{blue}{\bf{59.7}} & \textcolor{red}{\bf{86.5}} & \textcolor{blue}{\bf{46.9}} & \textcolor{red}{\bf{92.4}} 
            & \textcolor{blue}{\bf{56.0}} & 22.6 & \textcolor{blue}{\bf{60.0}} & \textcolor{blue}{\bf{61.2}}\\
            \bottomrule
            
            Ours1 & \textcolor{red}{\bf{72.3}}
            & 51.9 & 42.3 & \textcolor{red}{\bf{58.9}} & \textcolor{red}{\bf{90.7}} 
            & \textcolor{red}{\bf{77.8}} & \textcolor{red}{\bf{72.0}} & \textcolor{red}{\bf{89.9}} & \textcolor{red}{\bf{92.5}} 
            & \textcolor{red}{\bf{82.3}} & \textcolor{red}{\bf{85.6}} & \textcolor{red}{\bf{77.8}} & 43.8
            & 55.0 & \textcolor{blue}{\bf{64.3}} & \textcolor{red}{\bf{51.3}} & 67.5 
            & \textcolor{red}{\bf{95.7}} & 75.8 & \textcolor{red}{\bf{97.6}} & 50.5
            & \textcolor{red}{\bf{70.1}} & \textcolor{red}{\bf{88.6}} & \textcolor{blue}{\bf{65.2}} & \textcolor{red}{\bf{88.4}}\\
            
            Ours2 & \textcolor{red}{\bf{55.6}} 
            & ---- & 37.7 & ---- & ---- 
            & \textcolor{blue}{\bf{45.2}} & ---- & 51.1 & ---- 
            & \textcolor{red}{\bf{57.3}} & ---- & ---- & ---- 
            & ---- & ---- & \textcolor{red}{\bf{44.4}} & ---- 
            & \textcolor{red}{\bf{87.6}} & ---- & \textcolor{red}{\bf{48.1}} & ---- 
            & \textcolor{red}{\bf{61.1}} & \textcolor{red}{\bf{68.2}} & ---- & ----\\
            
            Ours3 & \textcolor{red}{\bf{52.3}} 
            & ---- & \textcolor{blue}{\bf{38.6}} & \textcolor{red}{\bf{57.1}} & ---- 
            & \textcolor{blue}{\bf{43.2}} & \textcolor{red}{\bf{57.8}} & 36.3 & \textcolor{red}{\bf{93.0}} 
            & \textcolor{red}{\bf{68.5}} & \textcolor{blue}{\bf{42.9}} & 39.5 & ---- 
            & ---- & \textcolor{red}{\bf{61.3}} & \textcolor{red}{\bf{33.1}} & ---- 
            & \textcolor{red}{\bf{83.4}} & ---- & 34.2 & ---- 
            & 39.0 & \textcolor{red}{\bf{40.5}} & \textcolor{red}{\bf{59.1}} & \textcolor{red}{\bf{62.1}}\\
            
            \midrule
            Avg & \textcolor{red}{\bf{63.0}} 
            & 51.9 & \textcolor{blue}{\bf{39.5}} & \textcolor{red}{\bf{58.0}} & \textcolor{red}{\bf{90.7}} 
            & \textcolor{red}{\bf{55.4}} & \textcolor{red}{\bf{64.9}} & 59.1 & \textcolor{red}{\bf{92.8}} 
            & \textcolor{red}{\bf{69.4}} & \textcolor{red}{\bf{64.3}} & \textcolor{blue}{\bf{58.7}} & 43.8
            & 55.0 & \textcolor{red}{\bf{62.8}} & \textcolor{red}{\bf{42.9}} & 67.5 
            & \textcolor{red}{\bf{88.9}} & 75.8 & \textcolor{red}{\bf{60.0}} & 50.5
            & \textcolor{red}{\bf{56.7}} & \textcolor{red}{\bf{65.8}} & \textcolor{red}{\bf{62.2}} & \textcolor{red}{\bf{75.3}}\\
            
			\bottomrule
		\end{tabular}
	\vspace{-3mm}
\end{table*}

\subsubsection{Comparison Results}
We first illustrate some visual results on ShapeNet in Fig.~\ref{fig:visual_comp}. Clear differences within the circulated regions can be observed and our result is much closer to the ground-truth.

We then list more detailed comparison on ShapeNet with some recent publications in 2018-2019 on ShapeNet in Table~\ref{tab:Shapenet_Result_per_class} that are also included in our summary in Fig. \ref{fig:ShapeNet} before. Clearly our approach achieves the best and second best in 6 and 4 out of 16 categories, respectively. Our class-mIoU and instance-mIoU results improve the state-of-the-art by 1.9\% and 0.6\%, respectively.

On PartNet our approach improves the state-of-the-art much more significantly, as listed in Table \ref{tab:PartNet}, with marginals of 6.8\%, 9.1\%, 5.9\% on the three different segmentation levels, leading to an average improvement of 3.2\%. Notice that PartNet is much more challenging given the number of categories as well as shapes. For instance, PointNet achieves 83.7\% on ShapeNet but only 57.9\% on PartNet. Our method, however, is much more robust and reliable with only 16.5 \% decrease. Taking into account all the categories in the three levels, \ie 50 in total, we achieve the best in 31 out of 50.


\section{Conclusion}
In this paper we address the problem of point cloud semantic segmentation by taking advantage of conventional 2D CNNs. To this end, we propose a novel segmentation pipeline, including graph construction from point clouds, graph-to-image mapping using graph drawing, and point segmentation using U-Net. The computational bottleneck in our pipeline is the graph drawing algorithm, which is essentially an integer programming problem. To accelerate the computation, we further propose a novel hierarchical approximate algorithm with complexity dominated by $O(n^{\frac{L+1}{L}})$, leading to a save of about 97\% running time. To better capture the local context embedded in our image presentations from point cloud, we also propose a multi-scale U-Net as our network. We evaluate our pipeline on ShapeNet and PartNet, achieving new state-of-the-art performance on both data sets with significantly large margins compared with the literature.

{
\bibliographystyle{ieee_fullname}
\bibliography{egbib}

\begin{thebibliography}{10}\itemsep=-1pt

\bibitem{amenta2007complexity}
Nina Amenta, Dominique Attali, and Olivier Devillers.
\newblock Complexity of delaunay triangulation for points on
  lower-dimensional\~{} polyhedra.
\newblock In {\em 18th ACM-SIAM Symposium on Discrete Algorithms}, pages
  1106--1113, 2007.

\bibitem{badrinarayanan2017segnet}
Vijay Badrinarayanan, Alex Kendall, and Roberto Cipolla.
\newblock Segnet: A deep convolutional encoder-decoder architecture for image
  segmentation.
\newblock {\em IEEE transactions on pattern analysis and machine intelligence},
  39(12):2481--2495, 2017.

\bibitem{bradley1977applied}
Stephen~P Bradley, Arnoldo~C Hax, and Thomas~L Magnanti.
\newblock {\em Applied mathematical programming}.
\newblock Addison-Wesley Publishing Company, 1977.

\bibitem{Buitinck2013sklearn}
Lars Buitinck, Gilles Louppe, Mathieu Blondel, Fabian Pedregosa, Andreas
  Mueller, Olivier Grisel, Vlad Niculae, Peter Prettenhofer, Alexandre
  Gramfort, Jaques Grobler, Robert Layton, Jake VanderPlas, Arnaud Joly, Brian
  Holt, and Ga{\"{e}}l Varoquaux.
\newblock {API} design for machine learning software: experiences from the
  scikit-learn project.
\newblock In {\em ECML PKDD Workshop: Languages for Data Mining and Machine
  Learning}, pages 108--122, 2013.

\bibitem{cai2018comprehensive}
Hongyun Cai, Vincent~W Zheng, and Kevin Chen-Chuan Chang.
\newblock A comprehensive survey of graph embedding: Problems, techniques, and
  applications.
\newblock {\em IEEE Transactions on Knowledge and Data Engineering},
  30(9):1616--1637, 2018.

\bibitem{caltagirone2017fast}
Luca Caltagirone, Samuel Scheidegger, Lennart Svensson, and Mattias Wahde.
\newblock Fast lidar-based road detection using fully convolutional neural
  networks.
\newblock In {\em 2017 ieee intelligent vehicles symposium (iv)}, pages
  1019--1024. IEEE, 2017.

\bibitem{chrobak1995linear}
Marek Chrobak and Thomas~H Payne.
\newblock A linear-time algorithm for drawing a planar graph on a grid.
\newblock {\em Information Processing Letters}, 54(4):241--246, 1995.

\bibitem{cui2018survey}
Peng Cui, Xiao Wang, Jian Pei, and Wenwu Zhu.
\newblock A survey on network embedding.
\newblock {\em IEEE Transactions on Knowledge and Data Engineering},
  31(5):833--852, 2018.

\bibitem{delaunay1934sphere}
Boris Delaunay et~al.
\newblock Sur la sphere vide.
\newblock {\em Izv. Akad. Nauk SSSR, Otdelenie Matematicheskii i Estestvennyka
  Nauk}, 7(793-800):1--2, 1934.

\bibitem{Duan_2019_SRN}
Yueqi Duan, Yu Zheng, Jiwen Lu, Jie Zhou, and Qi Tian.
\newblock Structural relational reasoning of point clouds.
\newblock In {\em CVPR}, June 2019.

\bibitem{frishman2007multi}
Yaniv Frishman and Ayellet Tal.
\newblock Multi-level graph layout on the gpu.
\newblock {\em IEEE Transactions on Visualization and Computer Graphics},
  13(6):1310--1319, 2007.

\bibitem{fruchterman1991graph}
Thomas~MJ Fruchterman and Edward~M Reingold.
\newblock Graph drawing by force-directed placement.
\newblock {\em Software: Practice and experience}, 21(11):1129--1164, 1991.

\bibitem{Hagberg2008networkx}
Aric~A. Hagberg, Daniel~A. Schult, and Pieter~J. Swart.
\newblock Exploring network structure, dynamics, and function using networkx.
\newblock In Ga\"el Varoquaux, Travis Vaught, and Jarrod Millman, editors, {\em
  Proceedings of the 7th Python in Science Conference}, pages 11 -- 15,
  Pasadena, CA USA, 2008.

\bibitem{hamilton2017representation}
William~L Hamilton, Rex Ying, and Jure Leskovec.
\newblock Representation learning on graphs: Methods and applications.
\newblock {\em arXiv preprint arXiv:1709.05584}, 2017.

\bibitem{Han_2019_MAP_VAE}
Zhizhong Han, Xiyang Wang, Yu-Shen Liu, and Matthias Zwicker.
\newblock Multi-angle point cloud-vae: Unsupervised feature learning for 3d
  point clouds from multiple angles by joint self-reconstruction and
  half-to-half prediction.
\newblock In {\em ICCV}, October 2019.

\bibitem{he2016resnet}
Kaiming He, Xiangyu Zhang, Shaoqing Ren, and Jian Sun.
\newblock Deep residual learning for image recognition.
\newblock In {\em CVPR}, pages 770--778, 2016.

\bibitem{hua2018pointwise}
Binh-Son Hua, Minh-Khoi Tran, and Sai-Kit Yeung.
\newblock Pointwise convolutional neural networks.
\newblock In {\em CVPR}, pages 984--993, 2018.

\bibitem{huang2018recurrent}
Qiangui Huang, Weiyue Wang, and Ulrich Neumann.
\newblock Recurrent slice networks for 3d segmentation of point clouds.
\newblock In {\em CVPR}, pages 2626--2635, 2018.

\bibitem{jacomy2014forceatlas2}
Mathieu Jacomy, Tommaso Venturini, Sebastien Heymann, and Mathieu Bastian.
\newblock Forceatlas2, a continuous graph layout algorithm for handy network
  visualization designed for the gephi software.
\newblock {\em PloS one}, 9(6):e98679, 2014.

\bibitem{Jiang_2019_Hierarchical}
Li Jiang, Hengshuang Zhao, Shu Liu, Xiaoyong Shen, Chi-Wing Fu, and Jiaya Jia.
\newblock Hierarchical point-edge interaction network for point cloud semantic
  segmentation.
\newblock In {\em ICCV}, October 2019.

\bibitem{kamada1989algorithm}
Tomihisa Kamada, Satoru Kawai, et~al.
\newblock An algorithm for drawing general undirected graphs.
\newblock {\em Information processing letters}, 31(1):7--15, 1989.

\bibitem{kingma2014adam}
Diederik~P Kingma and Jimmy Ba.
\newblock Adam: A method for stochastic optimization.
\newblock {\em arXiv preprint arXiv:1412.6980}, 2014.

\bibitem{klokov2017escape}
Roman Klokov and Victor Lempitsky.
\newblock Escape from cells: Deep kd-networks for the recognition of 3d point
  cloud models.
\newblock In {\em Proceedings of the IEEE International Conference on Computer
  Vision}, pages 863--872, 2017.

\bibitem{kobourov2012spring}
Stephen~G Kobourov.
\newblock Spring embedders and force directed graph drawing algorithms.
\newblock {\em arXiv preprint arXiv:1201.3011}, 2012.

\bibitem{Komarichev_2019_A_CNN}
Artem Komarichev, Zichun Zhong, and Jing Hua.
\newblock A-cnn: Annularly convolutional neural networks on point clouds.
\newblock In {\em CVPR}, June 2019.

\bibitem{landrieu2018large}
Loic Landrieu and Martin Simonovsky.
\newblock Large-scale point cloud semantic segmentation with superpoint graphs.
\newblock In {\em CVPR}, pages 4558--4567, 2018.

\bibitem{le2018pointgrid}
Truc Le and Ye Duan.
\newblock Pointgrid: A deep network for 3d shape understanding.
\newblock In {\em CVPR}, pages 9204--9214, 2018.

\bibitem{Lei_2019_Octree}
Huan Lei, Naveed Akhtar, and Ajmal Mian.
\newblock Octree guided cnn with spherical kernels for 3d point clouds.
\newblock In {\em CVPR}, June 2019.

\bibitem{li2018so_net}
Jiaxin Li, Ben~M Chen, and Gim Hee~Lee.
\newblock So-net: Self-organizing network for point cloud analysis.
\newblock In {\em CVPR}, pages 9397--9406, 2018.

\bibitem{li2018pointcnn}
Yangyan Li, Rui Bu, Mingchao Sun, Wei Wu, Xinhan Di, and Baoquan Chen.
\newblock Pointcnn: Convolution on x-transformed points.
\newblock In {\em Advances in Neural Information Processing Systems}, pages
  820--830, 2018.

\bibitem{li2019graph}
Zongmin Li, Jun Zhang, Guanlin Li, Yujie Liu, and Siyuan Li.
\newblock Graph attention neural networks for point cloud recognition.
\newblock In {\em 2019 IEEE International Conference on Multimedia and Expo
  (ICME)}, pages 387--392. IEEE, 2019.

\bibitem{liang2019hierarchical}
Zhidong Liang, Ming Yang, Liuyuan Deng, Chunxiang Wang, and Bing Wang.
\newblock Hierarchical depthwise graph convolutional neural network for 3d
  semantic segmentation of point clouds.
\newblock In {\em 2019 International Conference on Robotics and Automation
  (ICRA)}, pages 8152--8158. IEEE, 2019.

\bibitem{Liu_2019_DensePoint}
Yongcheng Liu, Bin Fan, Gaofeng Meng, Jiwen Lu, Shiming Xiang, and Chunhong
  Pan.
\newblock Densepoint: Learning densely contextual representation for efficient
  point cloud processing.
\newblock In {\em ICCV}, October 2019.

\bibitem{Liu_2019_RS_CNN}
Yongcheng Liu, Bin Fan, Shiming Xiang, and Chunhong Pan.
\newblock Relation-shape convolutional neural network for point cloud analysis.
\newblock In {\em CVPR}, June 2019.

\bibitem{long2015fully}
Jonathan Long, Evan Shelhamer, and Trevor Darrell.
\newblock Fully convolutional networks for semantic segmentation.
\newblock In {\em CVPR}, pages 3431--3440, 2015.

\bibitem{lyu2018chipnet}
Yecheng Lyu, Lin Bai, and Xinming Huang.
\newblock Chipnet: Real-time lidar processing for drivable region segmentation
  on an fpga.
\newblock {\em IEEE Transactions on Circuits and Systems I: Regular Papers},
  66(5):1769--1779, 2018.

\bibitem{lyu2018real}
Yecheng Lyu, Lin Bai, and Xinming Huang.
\newblock Real-time road segmentation using lidar data processing on an fpga.
\newblock In {\em 2018 IEEE International Symposium on Circuits and Systems
  (ISCAS)}, pages 1--5. IEEE, 2018.

\bibitem{Mao_2019_Interpolated}
Jiageng Mao, Xiaogang Wang, and Hongsheng Li.
\newblock Interpolated convolutional networks for 3d point cloud understanding.
\newblock In {\em ICCV}, October 2019.

\bibitem{Meng_2019_VV_Net}
Hsien-Yu Meng, Lin Gao, Yu-Kun Lai, and Dinesh Manocha.
\newblock Vv-net: Voxel vae net with group convolutions for point cloud
  segmentation.
\newblock In {\em ICCV}, October 2019.

\bibitem{meyerhenke2015drawing}
Henning Meyerhenke, Martin N{\"o}llenburg, and Christian Schulz.
\newblock Drawing large graphs by multilevel maxent-stress optimization.
\newblock In {\em International Symposium on Graph Drawing}, pages 30--43.
  Springer, 2015.

\bibitem{milioto2019rangenet++}
Andres Milioto, Ignacio Vizzo, Jens Behley, and Cyrill Stachniss.
\newblock Rangenet++: Fast and accurate lidar semantic segmentation.
\newblock In {\em Proc. of the IEEE/RSJ Intl. Conf. on Intelligent Robots and
  Systems (IROS)}, 2019.

\bibitem{Mo_2019_PartNet_dataset}
Kaichun Mo, Shilin Zhu, Angel~X. Chang, Li Yi, Subarna Tripathi, Leonidas~J.
  Guibas, and Hao Su.
\newblock Partnet: A large-scale benchmark for fine-grained and hierarchical
  part-level 3d object understanding.
\newblock In {\em CVPR}, June 2019.

\bibitem{pan2019pointatrousnet}
Liang Pan, Pengfei Wang, and Chee-Meng Chew.
\newblock Pointatrousnet: Point atrous convolution for point cloud analysis.
\newblock {\em IEEE Robotics and Automation Letters}, 4(4):4035--4041, 2019.

\bibitem{Pham_2019_JSIS3D}
Quang-Hieu Pham, Thanh Nguyen, Binh-Son Hua, Gemma Roig, and Sai-Kit Yeung.
\newblock Jsis3d: Joint semantic-instance segmentation of 3d point clouds with
  multi-task pointwise networks and multi-value conditional random fields.
\newblock In {\em CVPR}, June 2019.

\bibitem{qi2017pointnet}
Charles~R Qi, Hao Su, Kaichun Mo, and Leonidas~J Guibas.
\newblock Pointnet: Deep learning on point sets for 3d classification and
  segmentation.
\newblock In {\em CVPR}, pages 652--660, 2017.

\bibitem{qi2017pointnet++}
Charles~Ruizhongtai Qi, Li Yi, Hao Su, and Leonidas~J Guibas.
\newblock Pointnet++: Deep hierarchical feature learning on point sets in a
  metric space.
\newblock In {\em Advances in neural information processing systems}, pages
  5099--5108, 2017.

\bibitem{rao2019spherical}
Yongming Rao, Jiwen Lu, and Jie Zhou.
\newblock Spherical fractal convolutional neural networks for point cloud
  recognition.
\newblock In {\em CVPR}, pages 452--460, 2019.

\bibitem{ronneberger2015u}
Olaf Ronneberger, Philipp Fischer, and Thomas Brox.
\newblock U-net: Convolutional networks for biomedical image segmentation.
\newblock In {\em International Conference on Medical image computing and
  computer-assisted intervention}, pages 234--241. Springer, 2015.

\bibitem{schnyder1990embedding}
Walter Schnyder.
\newblock Embedding planar graphs on the grid.
\newblock In {\em Proceedings of the first annual ACM-SIAM symposium on
  Discrete algorithms}, pages 138--148. Society for Industrial and Applied
  Mathematics, 1990.

\bibitem{shen2018mining}
Yiru Shen, Chen Feng, Yaoqing Yang, and Dong Tian.
\newblock Mining point cloud local structures by kernel correlation and graph
  pooling.
\newblock In {\em CVPR}, pages 4548--4557, 2018.

\bibitem{shi2000normalized}
Jianbo Shi and Jitendra Malik.
\newblock Normalized cuts and image segmentation.
\newblock {\em Departmental Papers (CIS)}, page 107, 2000.

\bibitem{silberman2012indoor}
Nathan Silberman, Derek Hoiem, Pushmeet Kohli, and Rob Fergus.
\newblock Indoor segmentation and support inference from rgbd images.
\newblock In {\em European Conference on Computer Vision}, pages 746--760.
  Springer, 2012.

\bibitem{su2018splatnet}
Hang Su, Varun Jampani, Deqing Sun, Subhransu Maji, Evangelos Kalogerakis,
  Ming-Hsuan Yang, and Jan Kautz.
\newblock Splatnet: Sparse lattice networks for point cloud processing.
\newblock In {\em CVPR}, pages 2530--2539, 2018.

\bibitem{szegedy2015googlenet}
Christian Szegedy, Wei Liu, Yangqing Jia, Pierre Sermanet, Scott Reed, Dragomir
  Anguelov, Dumitru Erhan, Vincent Vanhoucke, and Andrew Rabinovich.
\newblock Going deeper with convolutions.
\newblock In {\em CVPR}, pages 1--9, 2015.

\bibitem{szegedy2015going}
Christian Szegedy, Wei Liu, Yangqing Jia, Pierre Sermanet, Scott Reed, Dragomir
  Anguelov, Dumitru Erhan, Vincent Vanhoucke, and Andrew Rabinovich.
\newblock Going deeper with convolutions.
\newblock In {\em CVPR}, pages 1--9, 2015.

\bibitem{te2018rgcnn}
Gusi Te, Wei Hu, Amin Zheng, and Zongming Guo.
\newblock Rgcnn: Regularized graph cnn for point cloud segmentation.
\newblock In {\em 2018 ACM Multimedia Conference on Multimedia Conference},
  pages 746--754. ACM, 2018.

\bibitem{Thomas_2019_KPConv}
Hugues Thomas, Charles~R. Qi, Jean-Emmanuel Deschaud, Beatriz Marcotegui,
  Francois Goulette, and Leonidas~J. Guibas.
\newblock Kpconv: Flexible and deformable convolution for point clouds.
\newblock In {\em ICCV}, October 2019.

\bibitem{Virtanen2019SciPy}
Pauli {Virtanen}, Ralf {Gommers}, Travis~E. {Oliphant}, Matt {Haberland}, Tyler
  {Reddy}, David {Cournapeau}, Evgeni {Burovski}, Pearu {Peterson}, Warren
  {Weckesser}, Jonathan {Bright}, St{\'e}fan~J. {van der Walt}, Matthew
  {Brett}, Joshua {Wilson}, K. {Jarrod Millman}, Nikolay {Mayorov}, Andrew
  R.~J. {Nelson}, Eric {Jones}, Robert {Kern}, Eric {Larson}, CJ {Carey},
  {\.I}lhan {Polat}, Yu {Feng}, Eric~W. {Moore}, Jake {Vand erPlas}, Denis
  {Laxalde}, Josef {Perktold}, Robert {Cimrman}, Ian {Henriksen}, E.~A.
  {Quintero}, Charles~R {Harris}, Anne~M. {Archibald}, Ant{\^o}nio~H.
  {Ribeiro}, Fabian {Pedregosa}, Paul {van Mulbregt}, and SciPy 1.~0
  {Contributors}.
\newblock {SciPy 1.0--Fundamental Algorithms for Scientific Computing in
  Python}.
\newblock {\em arXiv e-prints}, page arXiv:1907.10121, Jul 2019.

\bibitem{Wang_2019_Graph}
Lei Wang, Yuchun Huang, Yaolin Hou, Shenman Zhang, and Jie Shan.
\newblock Graph attention convolution for point cloud semantic segmentation.
\newblock In {\em CVPR}, June 2019.

\bibitem{wang2017cnn}
Peng-Shuai Wang, Yang Liu, Yu-Xiao Guo, Chun-Yu Sun, and Xin Tong.
\newblock O-cnn: Octree-based convolutional neural networks for 3d shape
  analysis.
\newblock {\em ACM Transactions on Graphics (TOG)}, 36(4):72, 2017.

\bibitem{wang2018sgpn}
Weiyue Wang, Ronald Yu, Qiangui Huang, and Ulrich Neumann.
\newblock Sgpn: Similarity group proposal network for 3d point cloud instance
  segmentation.
\newblock In {\em CVPR}, pages 2569--2578, 2018.

\bibitem{wang2019dynamic}
Yue Wang, Yongbin Sun, Ziwei Liu, Sanjay~E Sarma, Michael~M Bronstein, and
  Justin~M Solomon.
\newblock Dynamic graph cnn for learning on point clouds.
\newblock {\em ACM Transactions on Graphics (TOG)}, 38(5):146, 2019.

\bibitem{wolsey2014integer}
Laurence~A Wolsey and George~L Nemhauser.
\newblock {\em Integer and combinatorial optimization}.
\newblock John Wiley \& Sons, 2014.

\bibitem{wu2017squeezeseg}
Bichen Wu, Alvin Wan, Xiangyu Yue, and Kurt Keutzer.
\newblock Squeezeseg: Convolutional neural nets with recurrent crf for
  real-time road-object segmentation from 3d lidar point cloud.
\newblock {\em ICRA}, 2018.

\bibitem{wu2019squeezesegv2}
Bichen Wu, Xuanyu Zhou, Sicheng Zhao, Xiangyu Yue, and Kurt Keutzer.
\newblock Squeezesegv2: Improved model structure and unsupervised domain
  adaptation for road-object segmentation from a lidar point cloud.
\newblock In {\em ICRA}, 2019.

\bibitem{Wu_2019_PointConv}
Wenxuan Wu, Zhongang Qi, and Li Fuxin.
\newblock Pointconv: Deep convolutional networks on 3d point clouds.
\newblock In {\em CVPR}, June 2019.

\bibitem{xu2018spidercnn}
Yifan Xu, Tianqi Fan, Mingye Xu, Long Zeng, and Yu Qiao.
\newblock Spidercnn: Deep learning on point sets with parameterized
  convolutional filters.
\newblock In {\em Proceedings of the European Conference on Computer Vision
  (ECCV)}, pages 87--102, 2018.

\bibitem{yang2018foldingnet}
Yaoqing Yang, Chen Feng, Yiru Shen, and Dong Tian.
\newblock Foldingnet: Point cloud auto-encoder via deep grid deformation.
\newblock In {\em CVPR}, pages 206--215, 2018.

\bibitem{yi2016ShapeNet}
Li Yi, Vladimir~G Kim, Duygu Ceylan, I Shen, Mengyan Yan, Hao Su, Cewu Lu,
  Qixing Huang, Alla Sheffer, Leonidas Guibas, et~al.
\newblock A scalable active framework for region annotation in 3d shape
  collections.
\newblock {\em ACM Transactions on Graphics (TOG)}, 35(6):210, 2016.

\bibitem{yi2017syncspeccnn}
Li Yi, Hao Su, Xingwen Guo, and Leonidas~J Guibas.
\newblock Syncspeccnn: Synchronized spectral cnn for 3d shape segmentation.
\newblock In {\em CVPR}, pages 2282--2290, 2017.

\bibitem{Yu_2019_PartNet}
Fenggen Yu, Kun Liu, Yan Zhang, Chenyang Zhu, and Kai Xu.
\newblock Partnet: A recursive part decomposition network for fine-grained and
  hierarchical shape segmentation.
\newblock In {\em CVPR}, June 2019.

\bibitem{Zhang2019ShellNet}
Zhiyuan Zhang, Binh-Son Hua, and Sai-Kit Yeung.
\newblock Shellnet: Efficient point cloud convolutional neural networks using
  concentric shells statistics.
\newblock In {\em ICCV}, October 2019.

\bibitem{Zhao_2019_PointWeb}
Hengshuang Zhao, Li Jiang, Chi-Wing Fu, and Jiaya Jia.
\newblock Pointweb: Enhancing local neighborhood features for point cloud
  processing.
\newblock In {\em CVPR}, June 2019.

\end{thebibliography}
}

\end{document}